\documentclass{article}

\usepackage[final, nonatbib]{neurips_2024}

\usepackage[utf8]{inputenc} 
\usepackage[T1]{fontenc}    
\usepackage{amsmath}
\usepackage{amssymb}
\usepackage{amsthm}
\usepackage{bbm}
\usepackage{bm}  
\usepackage[british]{babel}
\usepackage{balance}
\usepackage{mathtools}
\usepackage{microtype}      
\usepackage{nicefrac}
\usepackage{physics}
\usepackage{csquotes}
\usepackage{hyperref}    
\usepackage{cleveref}
\usepackage{url}            
\usepackage{booktabs}       
\usepackage{amsfonts}       
\usepackage{xcolor}         
\usepackage{listings}
\usepackage[backref=true,sorting=none,doi=false,isbn=false,giveninits=true,style=nature,url=false]{biblatex} 
\bibliography{nn}

\newcommand{\lorenzo}[1]{\textcolor{C1}{LB:~#1}}

\definecolor{C0}{HTML}{1f77b4}
\definecolor{C1}{HTML}{ff7f0e}
\definecolor{C2}{HTML}{2ca02c}
\definecolor{C3}{HTML}{d62728}
\definecolor{C4}{HTML}{9467bd}
\definecolor{C5}{HTML}{8c564b}
\definecolor{C6}{HTML}{e377c2}
\definecolor{C7}{HTML}{7f7f7f}
\definecolor{C8}{HTML}{bcbd22}
\definecolor{C9}{HTML}{17becf}
\definecolor{gaussian}{HTML}{47908C}

\def\SignalToNoiseRatio{\beta}
\def\NumberOfTrainingSamples{n}
\def\InputDimension{d}

\def\NTK_alpha{\alpha_\text{NTK}}
\def\Q{\mathbb Q}
\def\P{\mathbb P}
\def\E{\mathbb E}
\def\N{\mathbb N}
\def\R{\mathbb R}
\def\amult{{\boldsymbol{\alpha}}}
\def\bmult{{\boldsymbol{\beta}}}
\def\ammult{{\underline{\boldsymbol{\alpha}}}}
\def\bmmult{{\underline{\boldsymbol{\beta}}}}
\def\deriv{{\text{d}}}
\def\ux{{\underline{x}}}
\def\uP{{\underline{\mathbb P}}}
\def\uQ{{\underline{\mathbb Q}}}

\def\id{\mathbbm{1}}

\newcommand{\EE}{\mathbb{E}\,}

\newcommand{\reals}{\mathbb{R}}

\usepackage{wrapfig}
\usepackage{comment}
\newcommand{\citet}[1]{\textcite{#1}}
\newcommand{\citep}[1]{\cite{#1}}

\newtheorem{theorem}{Theorem}

\newtheorem{lemma}[theorem]{Lemma}
\newtheorem{example}[theorem]{Example}
\newtheorem{proposition}[theorem]{Proposition}
\newtheorem{corollary}[theorem]{Corollary}
\newtheorem{definition}[theorem]{Definition}
\newtheorem{assumption}[theorem]{Assumption}
\newtheorem{conjecture}[theorem]{Conjecture}

\title{Learning from higher-order correlations, efficiently:\\
    hypothesis tests, random features, and neural networks}

\author{%
  Eszter {Sz\'ekely}\textsuperscript{\textdagger} \thanks{Current address: CCFE, Culham Science Centre, Abingdon, Oxon, OX14 3DB, UK} \\
  \And
  Lorenzo Bardone \thanks{These authors contributed equally.} \\
  \And
  Federica Gerace\thanks{Current address: Dipartimento di Matematica, Universita' di Bologna, 
  Bologna (BO), Italy} \\
  \AND
  Sebastian Goldt\thanks{Correspondence to: \{eszekely, lbardone, fgerace, sgoldt\}@sissa.it} \\[1em]
  International School of Advanced Studies (SISSA)\\
  Trieste, Italy\\
}

\begin{document}
\maketitle

\begin{abstract}
\noindent Neural networks excel at discovering statistical patterns in
high-dimensional data sets. In practice, higher-order cumulants, which quantify
the non-Gaussian correlations between three or more variables, are particularly
important for the performance of neural networks. But how efficient are neural
networks at extracting features from higher-order cumulants? We study this
question in the spiked cumulant model, where the statistician needs to recover a
privileged direction or ``spike'' from the order-$p\ge 4$ cumulants
of~$d$-dimensional inputs. %
We first discuss the fundamental statistical and
computational limits of recovering the spike by analysing the number of
 samples~$n$ required to strongly distinguish between inputs from the spiked
cumulant model and isotropic Gaussian inputs. 
Existing literature established the presence of a wide statistical-to-computational gap in this problem. We deepen this line of work by finding an exact formula for the likelihood ratio norm which proves that statistical
distinguishability requires $n\gtrsim d$ samples, while distinguishing the two
distributions in polynomial time requires $n \gtrsim d^2$ samples for a wide
class of algorithms, i.e.\ those covered by the low-degree conjecture. 
Numerical experiments show that neural networks do indeed learn to distinguish
the two distributions with quadratic sample complexity, while ``lazy'' methods
like random features are not better than random guessing in this regime. Our
results show that neural networks extract information from higher-order
correlations in the spiked cumulant model efficiently, and reveal a large gap in
the amount of data required by neural networks and random features to learn from
higher-order cumulants.

\end{abstract}


\section{Introduction}

Discovering statistical patterns in high-dimensional data sets is the key
objective of machine learning. In a classification task, the differences between
inputs in different classes arise at different statistical levels of the inputs:
two different classes of images will usually have different means, different
covariances, and different higher-order cumulants (HOCs), which describe the non-Gaussian
part of the correlations between three or more pixels. While differences in the mean and
covariance allow for rudimentary classification, \citet{refinetti2023neural}
recently highlighted the importance of HOCs for the performance of neural
networks: when they removed the HOCs per class of the CIFAR10 training set, the test accuracy of
various deep neural networks dropped by up to 65
percentage points. The importance of higher-order cumulants (HOCs) for
classification in general and the
performance of neural networks in particular raises some fundamental questions: what are the
fundamental limits of learning from HOCs, i.e.\ how many samples~$n$ (``sample complexity'') are
required to extract information from the HOCs of a data set? How
many samples are required when using a \emph{tractable} algorithm? And
how do neural networks and other machine learning methods like random features
compare to those fundamental limits?

In this paper, we study these questions by analysing a series of binary
classification tasks. In one class, inputs
$x\in\reals^d$ are drawn from a normal distribution with zero mean and identity
covariance. These inputs are therefore isotropic: the distribution of the
high-dimensional points projected along a unit vector in any direction
in $\reals^d$ is a standard Gaussian distribution. Furthermore, all the
higher-order cumulants (of order $p\ge3$) of the inputs are identically
zero. Inputs in the second class are also isotropic, except for one special
direction $u\in\reals^d$ in which their distribution is different. This
direction $u$ is often called a ``spike'', and it can be encoded in different
cumulants: for example, we could ``spike'' the covariance by drawing inputs from
a Gaussian with mean zero and covariance $\id + \beta u u^\top$; the
signal-to-noise ratio $\beta > 0$ would then control the variance of $\lambda =
\langle  u, x \rangle$. Likewise, we could spike a higher-order cumulant of the
distribution and ask: what is the minimal number of samples required for a
neural network trained with SGD to distinguish between the two input classes?
This simple task serves as a proxy for more generic tasks: a neural
network cannot be able to \emph{extract} information from a given cumulant if it
 cannot even recognise that it is different from an isotropic one.

We can obtain the fundamental limits of detecting spikes at different levels of
the cumulant hierarchy by considering the hypothesis test between a null
hypothesis (the isotropic multivariate normal distribution with zero mean and
identity covariance) and an alternative ``spiked'' distribution. We
can then compare the sample complexity of neural networks to the number of
samples necessary to distinguish the two distributions using unlimited
computational power, or efficiently using algorithms that run in polynomial
time. A second natural comparison for neural networks are random features or
kernel methods. Since the discovery of the neural tangent
kernel~\citep{jacot2018neural}, and the practical success of kernels derived
from neural networks~\citep{arora2019exact, geiger2020disentangling}, there has
been intense interest in establishing the advantage of neural networks with
\emph{learnt} feature maps over classical methods with \emph{fixed} feature
maps, like kernel machines. The role of higher-order correlations for the
relative advantage of these methods has not been studied yet.

In the following, we first introduce some fundamental notions around hypothesis
tests and in particular the low-degree method~\citep{barak2019nearly,
  hopkins2017power, hopkins2017bayesian, hopkins2018statistical}, which will be
a key tool in our analysis, using the classic spiked Wishart model for sample
covariance matrices~\citep{baik2005phase, potters2020first}. For spiked cumulants, the existing literature (see \cref{sec:related_work} for a complete discussion) establishes the presence of a wide statistical-to-computational gap in this model. Our \textbf{main contributions} are then as follows:

\begin{itemize}
\item We deepen the understanding of the statistical-to-computational gap for learning from higher-order correlations on the statistical side by showing that at \emph{unbounded computational power}, a number of samples \emph{linear} in the input dimension
is required to reach \emph{statistical distinguishability}. We prove this by explicitly computing the norm of the \emph{likelihood ratio}, see \cref{thm:LRspiked_cumulant}.
\item On the algorithmic side, SQ bounds and previous low-degree analyses~\cite{diakonikolas_pancakes,dudeja-hsu}, showed that the sample complexity of learning from HOCs is instead \emph{quadratic} for a wide class of  polynomial-time algorithms (\cref{sec:spiked-cumulant-ldlr}). Here, we provide a different, more direct proof of such a bound.
\item Using these fundamental limits on learning from HOCs as benchmarks, we  show numerically
  that neural networks learn the mixed cumulant model efficiently, while random
  features do not, revealing a large separation between the two methods
  (\cref{sec:experiments-wishart,sec:experiments-cumulant}).
\item We finally show numerically that the distinguishability in a simple model
  for images~\citep{ingrosso2022data} is precipitated by a cross-over in the
  behaviour of the higher-order cumulants of the inputs
  (\cref{sec:experiments-nlgp}).
\end{itemize}

\subsection{Further related work \label{sec:related_work}}

\paragraph{Detecting spikes in high-dimensional data} There is an enormous
literature on statistical-to-computational gaps in the detection and estimation
of variously structured principal components of high-dimensional data sets. This
problem has been studied for the Wigner and Wishart ensembles of random
matrices~\citep{baik2005phase, paul2007asymptotics, berthet2012optimal,
  berthet2013complexity, lesieur2015phase, lesieur2015mmse, perry2016optimality,
  krzakala2016mutual, dia2016mutual, miolane2018phase, lelarge2019fundamental,
  el2020fundamental} and for spiked tensors~\citep{richard2014statistical,
  hopkins2015tensor, montanari2015limitation, perry2016statistical,
  kim2017community, lesieur2017statistical, arous2020algorithmic,
  jagannath2020statistical, niles2022estimation}. For comprehensive reviews of
the topic, see refs.~\cite{zdeborova2016statistical, lesieur2017constrained,
  bandeira2018notes,
  kunisky2019notes}. 
While the samples in these models are often non-Gaussian, depending on the prior distribution
over the spike $u$, the spike appears already at the level of the covariance of
the inputs. Here, we instead study a high-dimensional data model akin to the one
used by
\citet{chuang2017scaling} to study online independent component
analysis~(ICA). This data model can be interpreted as having an additional
whitening step to pre-process the inputs, which is a common pre-processing step
in ICA~\citep{hyvarinen2000independent}, hence inputs have identity covariance
even when cumulants of order $p\ge3$ are spiked. \citet{chuang2017scaling}
proved the existence of a scaling limit for online ICA, but did not consider the
sample complexity of recovering the spike/distinguishing the distributions, which is the focus of this paper.

\paragraph{NGCA, Gaussian pancakes and  low-degree polynomials } 
Related models have been introduced under the name of \emph{Non-Gaussian Component Analysis} (NGCA) \cite{JMLR:v7:blanchard06a}, and studied from the point of view of \emph{statistical query} (SQ) complexity in a sequence of papers~\cite{bean2014non, vempala2012structure,pmlr-v75-tan18a,Goyal2018NonGaussianCA,mao2022optimal,damian2024computationalstatistical}. In particular \cite{diakonikolas_pancakes} nicknames \emph{Gaussian pancakes} a class of models that includes the \emph{spiked cumulant model} presented here. The SQ bounds found in \cite{diakonikolas_pancakes} and refined in the subsequent works (see \citet{diakonikolas2023sq} and references therein) could be used, together with SQ-LDLR equivalence \cite{brennan21:SQ-LDLRequivalence}, to provide estimates on low-degree polynomials. Finally, \cite{dudeja-hsu} proves very general bounds on LDLR norm; 
our \cref{thm:ldlr-cumulant} provides an alternative derivation of these bounds, in a setting that is closer to the experiments in section~4. 

\paragraph{Separation between neural networks and random features} The discovery
of the neural tangent kernel by \citet{jacot2018neural} and the
flurry of results on linearised neural networks~\citep{li2017convergence,
  jacot2018neural, arora2019exact, du2019gradient, li2018learning,
  allen-zhu2019convergence, bordelon2020spectrum, canatar2021spectral,
  nguyen2021tight} has triggered a new wave of interest in the differences
between what neural networks and kernel methods can learn efficiently. While
statistical separation results have been well-known for a long
time~\citep{bach2017breaking}, recent work focused on understanding the
differences between random features and neural networks \emph{trained by
  gradient descent} both with wide hidden layers~\citep{ghorbani2019limitations,
  ghorbani2020neural, chizat2020implicit, geiger2020disentangling,
  daniely2020learning, paccolat2021geometric} or even with just a few hidden
neurons~\citep{yehudai2019power, refinetti2021classifying}. At the heart of the
data models in all of these theoretical models is a hidden, low-dimensional
structure in the task, either in input space (for mixture classification) or in the form of single- or many-index target functions. The impact of higher-order
input correlations in mixture classification tasks on the separation between
random features and neural networks has not been directly studied yet.

\paragraph{Reproducibility} We provide code for all of our experiments, including routines to generate the
various synthetic data sets we discuss, on GitHub
\href{https://github.com/eszter137/data-driven-separation}{https://github.com/eszter137/data-driven-separation}.

\section{The data models} \label{sec:data_models}%

Throughout the paper, we consider binary classification tasks where high-dimensional inputs~$x^\mu=(x^\mu_i) \in \reals^d$ have labels $y^\mu=\pm 1$. The total number of training samples is $2n$, i.e.\ we have $n$
samples per class. For the class $y^\mu=-1$, inputs~$x^\mu=z^\mu$, where
$z^\mu\sim_{\mathrm{iid}}\mathcal{N}(0, \id_d)$ and $\mathcal{N}$ denotes the
normal distribution. For the class $y^\mu=1$ instead, we consider ``spiked'' input models as follows.

\paragraph{The Gaussian case} The simplest spiked model is the spiked Wishart
model from random matrix theory~\citep{baik2005phase, potters2020first}, in which
\begin{equation}
  \label{eq:wishart} 
   x^\mu = \sqrt{\frac \beta d} g^\mu u + z^\mu, \quad g^\mu \sim_{\mathrm{iid}} \mathcal{N}(0, 1),
\end{equation}
where $u=(u_i)$ is a $d$-dimensional vector with norm $\norm{u} = \sqrt d$ whose
elements are drawn element-wise i.i.d.~according to some probability
distribution $\mathcal{P}_u$. In this model, 
inputs are Gaussian and indistinguishable from white noise except in the direction of the  ``spike'' $u$, where they
have variance $1 + \beta$, where $\beta >0$ is the \emph{signal-to-noise ratio}. 
By construction, all higher-order cumulants are zero.

\paragraph{The spiked cumulant model} To study the impact of
HOCs, we draw inputs in the ``spiked'' class from the data model used by \citet{chuang2017scaling} to study online
independent component analysis~(ICA). First, we sample inputs 
\begin{equation}
    \label{eq:spiked-cumulant-prewhitening}
    \tilde x^\mu = \sqrt{\frac \beta d} g^\mu u + z^\mu, \qquad g^\mu \sim_{\mathrm{i.i.d.}} p_g,
\end{equation}
as in the Wishart model, but crucially sample the latent variables $g^\mu$ from
some non-Gaussian distribution $p_g(g^\mu)$, say, the Rademacher distribution;
see \cref{assumptions:g} for a precise statement. For any non-Gaussian
$p_g$, the resulting inputs~$\tilde x^\mu$ have a non-trivial fourth-order
cumulant proportional to $\kappa^g_{4} u^{\otimes
  4}$, 
where $\kappa^g_{4} \equiv \EE {(g^\mu)}^4 - 3 \EE {(g^\mu)}^2$ is the fourth cumulant of the distribution of $g^\mu$. 
We fix the mean and variance of $p_g$ to be zero and one, respectively, so the
covariance of inputs has a covariance matrix
$\Sigma = \id_d + \nicefrac{\beta}{d} u u^\top$. To avoid trivial detection of
the spike from the covariance, the key ingredient
of the spiked cumulant model is that we whiten the inputs in that class, so that
the inputs are finally given by
\begin{equation}
  \label{eq:whitening}
  x^\mu = S \tilde x^\mu, \qquad S = \id - \frac{\beta}{1 + \beta + \sqrt{1 + \beta}} \frac{u u^\top}{d},
\end{equation}
with the whitening matrix $S$ (see \cref{app:whitening}). Hence inputs $x^\mu$
are isotropic Gaussians in all directions except $u$, where they are a weighted
sum of $g^\mu$ and $\langle z, u \rangle$. The whitening therefore changes the
interpretation of $\beta$: rather than being a signal-to-noise ratio, as in the
spiked Wishart model, here $\beta$ controls the quadratic interpolation between
the distributions of $g^\mu$ and $z$ in the direction of $u$ (see
\cref{eq:xexpl} in the appendix). This leaves us with a data set where inputs in both classes have an average
covariance matrix that is the identity, which means that PCA or linear neural
networks~\citep{baldi1991temporal, le1991eigenvalues, saxe2014exact,
  advani2020high} cannot detect any difference between the two classes.

\section{How many samples do we need to learn?} \label{sec:sample_complexity}
 Given a data set sampled from the spiked cumulant model, we can now ask: how many
samples does a statistician need to reliably detect whether inputs are Gaussian
or not, i.e.~whether HOCs are spiked or not? This is equivalent to the
hypothesis test between $n$ i.i.d.\ samples of  the isotropic normal distribution in $\R^d$ as the null
hypothesis $\Q_{n,d}$, and $n$ i.i.d.\ samples of the spiked cumulant model \cref{eq:whitening} as the
alternative hypothesis $\P_{n,d}$. 
In \cref{sec:spiked-cumulant-lr} we will first consider the problem from a statistical point of view, assuming to have \emph{unbounded computational power} and no restrictions on the distinguishing algorithms. Then, in \cref{sec:spiked-cumulant-ldlr} we will use the \emph{low-degree method} to understand how the picture changes when we restrict to  algorithms whose running time is at most polynomial in the space dimension $d$.

\subsection{Statistical distinguishability: LR analysis} \label{sec:spiked-cumulant-lr}
Recall the notion of  \emph{strong asymptotic distinguishability}: two sequences of probability measures are strongly distinguishable if it is possible to design statistical tests that can classify correctly which of the two distributions a sample was drawn from with probabilities of type I and  II errors that converge to 0 (see \cref{sec:LDLRtheory} for the precise definition). Using this definition of distinguishability, we will ask what is the \emph{statistical sample complexity exponent}, i.e.\ the minimum $\theta$ such that in the high-dimensional limit $d\to \infty$, if $n\asymp d^\theta$, then $\P_{n,d}$ and $\Q_{n,d}$ are strongly distinguishable (with no constraints on the complexity of statistical tests).

A useful quantity to consider is the 
\emph{likelihood ratio} (LR) of probability measures, which is defined as
\begin{equation}
L(x):= \frac{\deriv \P}{\deriv \Q}(x).
\end{equation}
Computing the LR norm $||L||^2:=\E_{\Q}[L^2]$ is an excellent tool to probe for distinguishability: if $(\P_{n,d})$ and $(\Q_{n,d})$ are strongly distinguishable, then $\|L_{n,d}\|\to \infty$ (this is the well-known \emph{second moment method for distinguishability}, see \cref{prop:SMMD} in the appendix for the precise statement). 
In the following we will apply this method, finding a formula for the LR norm  and then study its limit as a function of $\theta$.
\begin{figure}
  \centering
  \includegraphics[width=0.5\textwidth]{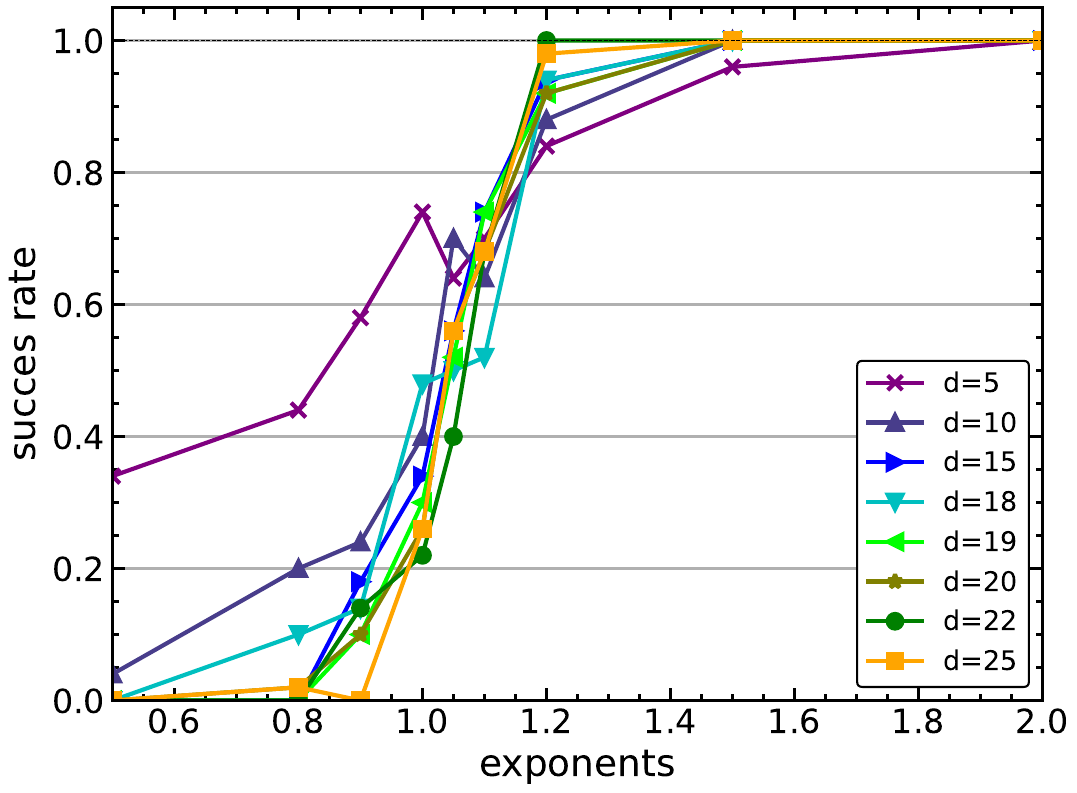}
  \caption{\label{fig:brute-force} \textbf{The performance of an exhaustive-search algorithm corroborates the presence of a phase transition
      for $\theta=1$, as suggested by \cref{thm:LRspiked_cumulant}.} Success rate of an exponential-time search algorithm over all the possible spikes in the $d$-hypercube as a function of the exponent $\theta$ that
    quantifies as $n=d^\theta$ the samples used in the log-likelihood test \eqref{eq:brute-force_estimator}, in the $g\sim$Radem$(1/2)$ case.}
\end{figure}
Here and throughout, we will denote the data matrix by $\ux=\left(x^{\mu}\right)^{\top}_{1,\dots,n}$; in general matrices of size $n \times d$ will be denoted with
underlined letters; see \cref{sec:notation} for a complete summary of our
notation. We will use Hermite polynomials, denoted by $(h_k)_k$, see \cref{subsec:hermite} for details. We assume that the spike $u$ is drawn from a known prior
$\mathcal P(u)$, and that the scalar latent variables
$\left(g^\mu\right)_{\mu=1,\dots,n}$ are drawn i.i.d.\ from a distribution~$p_g$
with the following properties:
\begin{assumption}[Assumption on latent variables $g^\mu$]%
  \label{assumptions:g} We assume that the one-dimensional probability
  distribution of the non-Gaussian random variable $p_g(g)$ is an even
  distribution, $p_g(g=\deriv x)=p_g(-g=-\deriv x)$, with mean 0 and variance 1,
  and that it satisfies the following requirement on the growth of its Hermite
  coefficients:
  \begin{equation}
    \label{eq:assumpHermCoef}
    \mathbb{E} \left[h_m(g)\right]\le \Lambda^m m!
  \end{equation}
  where $\Lambda>0$ is a positive constant that does not depend on $m$.
  Finally, we assume that $p_g$ has tails that cannot go to 0 slower than a
  standard Gaussian, $\E[\exp(g^2/2)]<+\infty$.
\end{assumption}
A detailed discussion of these assumptions can be found in \cref{subsec:LDLRdetails}, as well as a proof that they are
satisfied by a wide class of distributions including all the compactly supported distributions with mean 0 and variance 1 (some concrete examples are $p_g$=Rademacher$(1/2)$ or
$p_g=\text{Unif}(-\sqrt 3,\sqrt 3)$).

Under these assumption we can  compute the following formula for the LR norm.
\begin{theorem} \label{thm:LRspiked_cumulant} Suppose that $u$ has $i.i.d.$\
  Rademacher prior and that the non-Gaussian distribution $p_g$ satisfies \cref{assumptions:g}. Then the norm of the total LR is given by
\begin{equation}\label{eq:LRnormf}
     \|L_{n,d}\|^{2}=\sum_{j=0}^d \binom{d}{j}\frac{1}{2^d}f\left(\beta,\frac {2j}d-1\right)^n,
\end{equation}
where $f$ is defined as the following average over  two independent replicas  $g_u,g_v\sim g$ of $g$:
\begin{equation}
  \label{eq:fdef}
    f(\beta,\lambda):=\underset{g_u,g_v}{\mathbb E}\left[\frac{1+\beta}{\sqrt{(1+\beta)^2-\beta^2\lambda^2}}e ^{-\frac{(1+\beta)\left((1+\beta)(g_u^2+g_v^2)-2\beta(g_ug_v)\lambda\right)}{2(1+\beta)^2-2\beta^2\lambda^2}+\frac{g_u^2+g_v^2}2}\right]
\end{equation}
\end{theorem}
We prove \cref{thm:LRspiked_cumulant} in \cref{subsec:LRanalysisappendix}. The theorem has key consequences in two directions:
\begin{itemize}
\item on the one hand, \cref{eq:LRnormf} implies that the LR norm is bounded for
  $\theta<1$ (see \cref{lemma:theta<1} in \cref{subsubsec:Radem_non_gaus}),
  which confirms that below that threshold \emph{strong distinguishability is
    impossible};
    \item on the other hand, \cref{eq:LRnormf} implies that whenever there exists $\tilde \lambda$ such that $f(\beta,\tilde \lambda)>1$,  we find
that $
    \norm{L_{n,d}}\ge \frac{f(\beta,\tilde \lambda)^n}{2^d}.$
Thus, the LR norm diverges as soon as $n$ grows as
$n\asymp d^\theta$ with any $\theta>1$. In
\cref{subsubsec:Radem_non_gaus} we detail as an example the case in which $g\sim $Rademacher($\nicefrac{1}{2}$) where the norm $\norm{L_{n,d}}$ even diverges for $\theta=1$
and $d \asymp \gamma n$ for some $\gamma>0$.
\end{itemize}
So, besides the intrinsic value of providing an exact formula for the LR of the
spiked cumulant model, \cref{thm:LRspiked_cumulant} implies the presence of a
phase transition at $\theta=1$ for the \emph{strong statistical
  distinguishability}.

A complementary approach that substantiates the presence of the statistical  phase transition at $\theta=1$ can be seen in \cref{fig:brute-force}, where we perform a maximum
    log-likelihood test along
    $u\cdot x^\mu$ for \emph{all} the possible spikes $u$ in the $d$-dimensional
    hypercube using the formula for the LR conditioned on the spike,
    \begin{equation} \label{eq:brute-force_estimator}
 \sum_{\mu=1}^n  \log\left(  \frac{p_x(x^\mu|u)}{p_z(x^\mu)}\right)= \sum_{\mu=1}^n  \log\underset{g}{\mathbb E}\sqrt{1+\beta} \exp\left(-\frac{1+\beta}{2}\left(g-\sqrt{\frac \beta {(1+\beta)d}}x^\mu\cdot u\right)^2+\frac{g^{2}}{2}\right),\end{equation}
(see \cref{eq:brute-force_estimatorappendix} in \cref{sec:cumulant-lr-details} for the derivation of this equation) and output the most
    likely $u$. Note that due to the exponential complexity in $d$ of the
    algorithm, it is unfeasible to reach large values for this
    parameter. However, even at small $d$ values, the success rate for
    retrieving the correct spike has a very steep increase at around
    $\theta=1$, as predicted by our analysis of the LR norm in
    \cref{thm:LRspiked_cumulant}. 

\subsection{Computational distinguishability: LDLR analysis}%
\label{sec:spiked-cumulant-ldlr}

We now compare the statistical
threshold of \cref{sec:spiked-cumulant-lr} with the \emph{computational sample complexity exponent} that
quantifies the sample complexity of detecting non-Gaussianity with an efficient algorithm that runs in a time that is polynomial in the input
dimension. The algorithmic sample complexity can be analysed rigorously for a wide class of algorithms
using the \emph{low-degree method}~\citep{barak2019nearly, hopkins2017power,
  hopkins2017bayesian, hopkins2018statistical, kunisky2019notes}. The low-degree
method arose from the study of the sum-of-squares
hierarchy~\citep{barak2019nearly} and rose to prominence when
\citet{hopkins2017bayesian} demonstrated that the method can capture the
Kesten--Stigum threshold for community detection in the stochastic block
model~\citep{kesten1966additional, decelle2011inference,
  decelle2011asymptotic}. In the case of hypothesis testing, the key quantity is
the \emph{low-degree likelihood ratio} (LDLR)~\citep{hopkins2018statistical,
  kunisky2019notes}.

\begin{definition}[Low-degree likelihood ratio (LDLR)]%
  Let $D\ge0$ be an integer. The low-degree likelihood ratio of degree $D$ is
  defined as
  \begin{equation}
    \label{eq:ldlr}
    L^{\le D} := \mathcal{P}^{\le D} L
  \end{equation}
  where $\mathcal{P}^{\le D}$ projects $L$ onto the space of polynomials of
  degree up to $D$, parallel to the Hilbert space structure defined by
  the scalar product
  $\langle f,g\rangle_{L^2(\Omega,\Q)}:= \E_{x\sim \Q} [f(x)g(x)]$.
\end{definition}

The idea of this method is that among degree-$D$ polynomials, $L^{\le D}$
captures optimally the difference between $\P$ and $\Q$, and this difference can
be quantified by the norm
$\norm{ L^{\le D}}=\norm{L^{\le D}}_{L^2(\Omega_n,\Q_n)}$. 
Hence in analogy to the \emph{second moment method} used in \cref{sec:spiked-cumulant-lr}, we can expect low-degree polynomials to be able to distinguish~$\uP$ from~$\uQ$ only when $\norm{L_n^{\le D(n)}}\underset{n}{\to} \infty$,
where $D(n)$ is a monotone sequence diverging with~$n$.  Indeed, the
following (informal) conjecture from \citet{hopkins2018statistical} states that this
criterion is valid not only for polynomial tests, but for all polynomial-time
algorithms:

\begin{conjecture}
  \label{conjecture:LDLR}
  For two sequences of measures $\Q_N, \P_N$ indexed by $N$, suppose that (i)
  $\Q_N$ is a product measure; (ii) $\P_N$ is \emph{sufficiently} symmetric with
  respect to permutations of its coordinates; and (iii) $\P_N$ is \emph{robust}
  with respect to perturbations with small amount of random noise. If
  $\|L_N^{\le D}\|=O(1)$ as $N\to \infty$ and for
  $D\ge (\log N)^{1+\varepsilon}$, for some $\varepsilon>0$, then there is no
  polynomial-time algorithm that strongly distinguishes the distributions $\Q$
  and $\P$.
\end{conjecture}
Even though this conjecture is still not proved in general, its empirical
confirmation on many benchmark problems has made it an important tool to probe
questions of computational complexity, see  \cref{thm:ldlr-wishart} for a simple example. 

We will now compute LDLR norm estimates for the spiked cumulant model, so that
the application of \cref{conjecture:LDLR} will help to understand the
\emph{computational sample complexity} of this model.
\begin{theorem}[LDLR for spiked cumulant model]%
  \label{thm:ldlr-cumulant}
  Suppose that $(u_i)_{i=1,\dots,d}$ are drawn i.i.d.~from the symmetric
  Rademacher distribution and that the non-Gaussian distribution $p_g$ satisfies \cref{assumptions:g}.
   Let $0<\varepsilon<1$ and
  assume $D(n)\asymp\log^{1+\varepsilon}(n)$. Take $n,d\to \infty$, with the
  scaling $n\asymp d^\theta$ for $\theta>0$. The following bounds hold:
  \begin{equation}
    \label{eq:LDLR_asymp_lower_bound}
    \norm{L_{n,d}^{\le D(n)}}^2 \ge\left( \frac 1 {\left\lfloor D(n)/4\right \rfloor}\left(\frac{\beta^2\kappa^{g}_4}{(1+\beta)^2}\right)^{2}\frac{n}{d^2}\right)^{\left\lfloor D(n)/4\right \rfloor}
  \end{equation}
  \begin{equation}
    \label{eq:asymp_upper_bound}
    \norm{L_{n,d}^{\le D(n)}}^2 \le
    1+\sum_{m=1}^{D(n)}\left(\frac{\Lambda^2\beta}{1+\beta}\right)^m
    m^{4m}\left(\frac{n}{d^2}\right)^{m/4}
  \end{equation}
  Taken together, \cref{eq:LDLR_asymp_lower_bound} and \cref{eq:asymp_upper_bound} imply
  the presence of a critical regime for $\theta_c=2$, and describe the behaviour
  of $\norm{L_{n,d}^{\le D}}$ for all $\theta \ne \theta_c$
  \begin{equation} \label{eq:LDLR_final_asympt} \lim_{n,d\to
      \infty}\norm{L_{n,d}^{\le D(n)}}=\begin{cases}
      1 & 0<\theta< 2 \\
      +\infty & \theta> 2
    \end{cases}
  \end{equation}
\end{theorem}
This theorem could be derived with different constants from lemma 26 and
proposition 8 in \citet{dudeja-hsu}. Here we also provide a different, more
direct argument. We sketch the proof of \cref{thm:ldlr-cumulant} in
\cref{sec:cumulant-ldlr-sketch} and give the complete proof in
\cref{sec:cumulant-ldlr-details}. We will discuss next the implications of the
results presented in \cref{sec:spiked-cumulant-lr} and
\cref{sec:spiked-cumulant-ldlr}


\subsection{Statistical-to-computational gaps in the
  spiked cumulant model}

Put together, our results for the statistical and computational sample
complexities of detecting non-Gaussianity in \cref{thm:LRspiked_cumulant} and
\cref{thm:ldlr-cumulant} suggest the existence of three different regimes in the
spiked cumulant model as we vary the exponent $\theta$, with a
statistical-to-computational gap: for $0\le \theta<1$, the problem is \emph{statistically impossible} in the sense that no
  algorithm is able to strongly distinguish $\P$ and $\Q$ with so few samples, since the LR norm is bounded. For $1<\theta<2$, the norm of the likelihood ratio with
  $g\sim \text{Rademacher}(1/2)$ diverges (even at~$\theta=1$ for some values of
  $\beta$), the problem could be statistically solvable (as validated by the results of exhaustive-search algorithms in \cref{fig:brute-force}), but \cref{conjecture:LDLR} suggests no polynomial-time
  algorithm is able to achieve distinguishability in this regime; this is the so-called \emph{hard phase}.
 If $\theta>2$, the problem is solvable in polynomial time by evaluating a
  polynomial function (fourth-order in each sample) and thresholding; this is the \emph{easy phase}.

The spiked cumulant model leads thus to intrinsically harder classification problems than  the spiked Wishart model, where the critical regime
is at $\theta=1$. The proof of \cref{thm:ldlr-cumulant} reveals
that this increased difficulty is a direct consequence of the whitening of the
inputs in \cref{eq:whitening}. Without whitening, degree-2 polynomials would also
give contributions to the LDLR~\eqref{eq:normLDLRgen} which would yield linear
sample complexity. The difference in sample complexity of the spiked Wishart and
spiked cumulant models mirrors the gap between the sample complexity of the best-known algorithms for matrix factorisation, which require linear sample
complexity, and tensor PCA for rank-1 spiked tensors of order
$k$~\citep{richard2014statistical, perry2016statistical, lesieur2017statistical}, where sophisticated spectral algorithms can
match the computational lower bound of~$d^{k/2}$ samples.



\section{Learning from HOCs with neural networks and random features}%
\label{sec:experiments}

The analysis of the (low-degree) likelihood ratio has given us a detailed
picture of the statistical and computational complexity of extracting information from the
covariance or the higher-order cumulants of data. We will now use these
fundamental limits to benchmark the sample complexity of two-layer neural networks (2LNN) trained with stochastic gradient descent on a binary discrimination task, where inputs in one class are drawn from the normal distribution $\mathcal{N}(0, \id_d)$, while inputs in the other class are drawn from the spiked Wishart or the spiked cumulant model. In addition, we will also benchmark random
feature methods (RF)~\citep{balcan2006kernels, rahimi2008random, rahimi2009weighted} as a finite-dimensional approximation of kernel methods~\cite{balcan2006kernels, rahimi2008random, rahimi2009weighted}.

In a nutshell, the idea behind our experiments is to first \emph{validate}
both 2LNN and RF on the simpler spiked Wishart task and then to apply both methods to the spiked cumulant model, 
where inputs are generated in a way that mirrors the spiked Wishart: comparing \cref{eq:wishart} with \cref{eq:spiked-cumulant-prewhitening,eq:whitening}, we see that the only differences are the whitening, and the latent distribution $p_g$. In our experiments, we choose the latent variables to be standard Gaussian for spiked Wishart, and Rademacher for the spiked cumulant -- hence the latent variables have matching first and second moments. However, we will see that the
spiked cumulant model exhibits a large gap in the sample complexity required for
neural networks or random features to learn the problem. We relegate details on the experimental setups such as hyper-parameters to \cref{app:experimental-methods}.

\subsection{Spiked Wishart model}%
\label{sec:experiments-wishart}

We  trained \textbf{two-layer ReLU neural networks}
$\phi_\theta(x) = v^\top \max\left(0, w^T x\right)$ with $m=5d$ hidden neurons
on the spiked Wishart classification task. We show the early-stopping test
accuracy of the networks in the linear and quadratic regimes with
$n_\mathrm{class} \asymp d,d^2$ samples per class in \cref{fig:spiked_wishart}A
and B, resp. Neural networks are able to solve this task, in the sense
that their classification error \emph{per sample} is below~0.5, implying strong
distinguishability of the whole data matrix. 
Indeed, some of the hidden neurons
converge to the spike $u$ of the covariance matrix, as can be seen from the
maximum value of the normalised overlaps
$\max_k w_k^\top u / \sqrt{\norm{w_k} \norm{u}}$ in \cref{fig:spiked_wishart}C
and D, where $w_k$ is the weight vector of the $k$th hidden neuron. In the
linear regime (C), there is an increasingly clear transition from a random
overlap for small data sets to a large overlap at large data sets as we increase
the input dimension $d$; in the
quadratic regime (D), the neural network recovers the spike almost perfectly. 

The \textbf{relatively large overlap between hidden neurons and spike at small sample
complexities} (\cref{fig:spiked_wishart}C and D) is due to the fact that we plot the maximum overlaps over a
relative large number of hidden neurons $m=5d$; hence even at initialisation, a
few neurons will have large overlaps. We verified that ensuring an initial
overlap of only $1 / \sqrt d$ by explicit orthogonalisation did not change our
results on distinguishability, see \cref{fig:small-initial-overlap}. A possible
explanation is that the dynamics of the wide network is dominated by the
majority of neurons, which do not have a macroscopic overlap.

\begin{figure*}[t!]
  \centering
  \includegraphics[trim=0 580 0 0,clip,width=\linewidth]{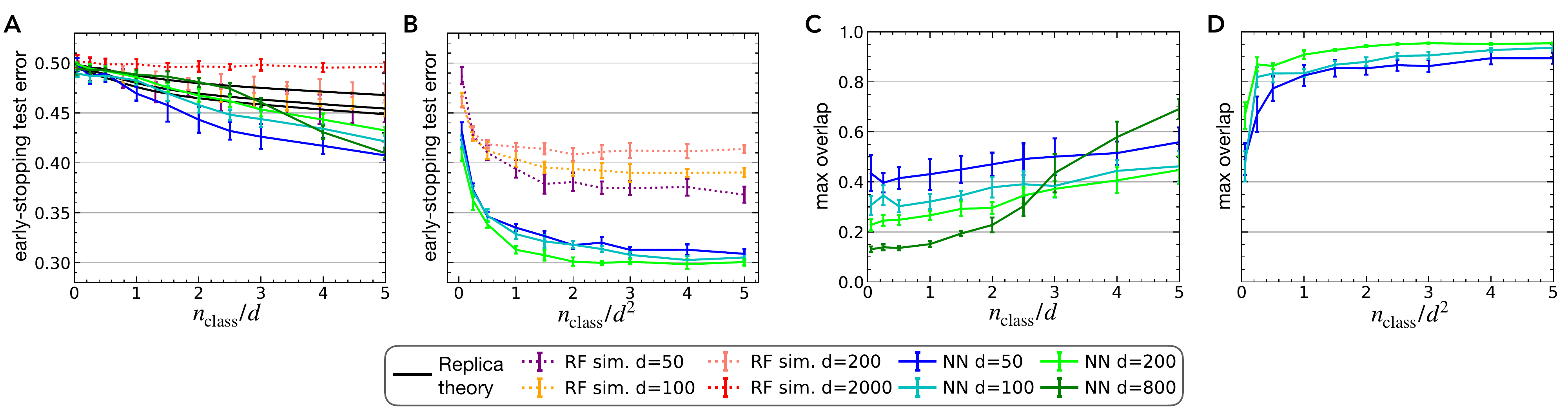}
  \caption{\label{fig:spiked_wishart} \textbf{Learning the spiked Wishart task
      with neural networks and random features.}  \textbf{(A,B)} Test accuracy
    of random features (RF) and early-stopping test accuracy of two-layer ReLU
    neural networks (NN) on the spiked Wishart task, \cref{eq:wishart}, with
    linear and quadratic sample complexity ($n_\mathrm{class}\asymp~d,~d^2$,
    respectively, where $d$ is the input
    dimension). 
    Predictions for the performance of random features obtained using replicas
    are shown in black. \textbf{(C,D)}~Maximum normalised overlaps of the
    networks' first-layer weights with the spike $u$, \cref{eq:wishart}.
    \emph{Parameters}: $\beta=5$. Neural nets and random features have $m=5d$
    hidden neurons. Full experimental details in
    \cref{app:experimental-methods}.}
\end{figure*}

Meanwhile, we found that \textbf{the performance of random features} tends to random
guessing at linear sample complexity, where we let the input dimension
$d\to\infty$ with $m/d$ and $n_\mathrm{class} / d$ fixed, while at quadratic
sample complexity, random features learn the task, although they perform worse
than neural networks. The failure of RF in the linear regime makes sense in
light of recent results that suggest that random features in this scaling regime
are limited to learning a linear approximation of the target
function~\citep{ghorbani2021linearized, xiao2022precise,
  mei2022hypercontractivity, misiakiewicz2023six}, while the LDLR analysis
\cref{sec:wishart} shows that the target function, i.e.\ the low-degree
likelihood ratio, is a quadratic polynomial. However, these results are, to the
best of our knowledge, restricted to the case of Gaussian isotropic inputs. 

To ensure that the performance of random features does indeed tend to random
guessing, we performed a \textbf{replica analysis} following
\citet{loureiro2021learninggaussians} for mixture classification tasks together
with the Gaussian equivalence theorem~\citep{goldt2020modeling, gerace2020generalisation,
  hu2022universality, mei2022generalization, goldt2022gaussian} (black lines in
\cref{fig:spiked_wishart}A, details in \cref{app:replicas}). Replica theory
perfectly predicts the performance of RF we obtain in numerical experiments
(red-ish dots) for various values of $d$ at linear sample complexity. We thus
find a clear separation in the sample complexity required by random features
($n_{\mathrm{class}} \gtrsim d^2$) and neural networks
($n_{\mathrm{class}} \gtrsim d$) to learn the spiked Wishart task. The replica analysis can be extended to the polynomial regime by a simple rescaling of the free energy~\citep{dietrich1999statistical} on several data models, like the vanilla teacher--student
setup~\citep{gardner1989three}, the Hidden manifold
model~\citep{goldt2020modeling}, and the vanilla Gaussian
mixture classification (see \cref{fig:RF_sample_regimes}). However, we found
that for the spiked Wishart model, the Gaussian equivalence theorem which we
need to deal with the random feature distribution \emph{fails} at quadratic
sample complexity. This might be due to the fact that in this case, the spike
induces a dominant direction in the covariance of the random features, and this
direction is also key to solving the task, which can lead to a breakdown of Gaussian equivalence~\citep{goldt2022gaussian}. A similar breakdown of the Gaussian equivalence theorem
at quadratic sample complexity has been demonstrated recently in teacher--student setups
by \citet{cui2023bayes} and \citet{camilli2023fundamental}.

\subsection{Spiked cumulant}%
\label{sec:experiments-cumulant}

Having thus validated the performance of 2LNN and RF on the simpler spiked
Wishart model, we turn to the spiked cumulant model and to the question: can
neural networks learn higher-order correlations, efficiently? The LDLR analysis
predicts that a polynomial-time algorithm requires at least a quadratic number
of samples to detect non-Gaussianity, and hence to solve the classification
task, and we found indeed that neural networks require at least quadratic sample
complexity to solve the task, \cref{fig:spiked_cumulant}A and~B. The high values
of the maximum overlap between hidden neurons and cumulant spike in the linear
regime (compared to $d^{-1/2}$) are again a consequence of choosing the maximum
overlap among $m=5d$ hidden neurons. Random features cannot solve this task even
at quadratic sample complexity, since they are limited to a quadratic
approximation of the target function~\citep{hu2022sharp, xiao2022precise,
  mei2022hypercontractivity, misiakiewicz2023six}, but we know from the LDLR
analysis that the target function is a fourth-order polynomial. We thus find an
even larger separation in the minimal number of samples required for random
features and neural networks to solve tasks that depend on directions which are
encoded exclusively in the higher-order cumulants of the inputs.

\begin{figure*}[t!]
  \centering
  \includegraphics[trim=0 70 0 560,clip,width=\linewidth]{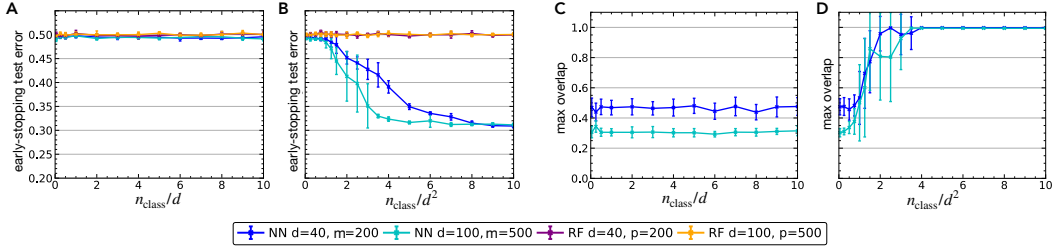}
  \caption{\label{fig:spiked_cumulant} \textbf{Learning the spiked cumulant task
      with neural networks and random features.}  \textbf{(A, B)} Test accuracy
    of random features (RF) and early-stopping test accuracy of two-layer ReLU
    neural networks (NN) on the spiked cumulant task \cref{eq:whitening} with
    linear and quadratic sample complexity ($n_\mathrm{class}\asymp~d,~d^2$,
    respectively, where $d$ is the input
    dimension). 
    \textbf{(C, D)}~Maximum normalised overlaps of the networks' first-layer
    weights with the spike $u$, \eqref{eq:whitening}.  \emph{Parameters}:
    $\beta=10$. Neural nets and random features have $m=5d$ hidden neurons, same
    optimisation as in \cref{fig:spiked_wishart}. Full experimental details in
    \cref{app:experimental-methods}.}
\end{figure*}
\subsection{Phase transitions and neural network performance in a simple model
  for images}%
\label{sec:experiments-nlgp}

We finally show another example of a
separation between the performance of random features and neural
networks in the feature-learning regime on a toy model for images that was
introduced recently by \citet{ingrosso2022data}, the non-linear Gaussian process
(NLGP). The idea is to generate inputs that are (i)~translation-invariant and
that (ii)~have sharp edges, both of which are hallmarks of natural
images~\citep{bell1996edges}. We first sample a vector $z\in\reals^d$ from a normal distribution with zero mean and covariance
$C_{ij} = \EE z_i z_j = \exp(-|i-j| / \xi)$ to ensure
translation-invariance of the inputs, with length scale~$\xi > 0$. We then introduce edges, i.e.\
sharp changes in luminosity, by passing $z$
through a saturating non-linearity like the error function, $x_i = \mathrm{erf} (g z_i ) / Z(g)$, 
where $Z(g)$ is a normalisation factor that ensures that the pixel-wise variance
$\EE x_i^2=1$ for all values of the gain $g>0$. The classification task is to discriminate these
``images'' from Gaussian inputs with the same mean and covariance, as illustrated in two dimensions in \cref{fig:nlgp}A. This task is different from the spiked cumulant model in that
the cumulant of the NLGP is not low-rank, so there are many directions that
carry a signal about the non-Gaussianity of the inputs.

We trained wide two-layer neural networks on this task and interpolated between
the feature-learning and the ``lazy'' regimes using the $\alpha$-renormalisation
trick of \citet{chizat2019lazy}. As we increase~$\alpha$, the networks go from
feature-learners ($\alpha=1$) to an effective random feature model and require
an increasing amount of data to solve the task, \cref{fig:nlgp}B. There
appears to be a sharp transition from random guessing to non-trivial performance
as we increase the number of training samples for all values of~$\alpha$. This  transition is preceded by a transition in the behaviour of
the higher-order cumulant that was reported by
\citet{ingrosso2022data}. They showed numerically that the CP-factors of the
empirical fourth-order cumulant $T$, defined as the vectors~$\hat u\in\reals^d$
that give the best rank-$r$ approximation
$\hat{T} = \sum_{k=1}^{r} \gamma_k \hat{u}_k^{\otimes 4}$ of
$T$~\citep{kolda2009tensor}, localise in space if the data set from which the
empirical cumulant is calculated is large enough. Quantifying the
localisation of a weight vector $w$ using the \emph{inverse
  participation ratio} 
\begin{equation}
  \label{eq:ipr}
  \mathrm{IPR}(w)= \frac{\sum_{i=1}^d w_i^4}{{\left(\sum_{i=1}^d w_i^2\right)}^2},
\end{equation}
we confirm that the leading CP-factors of the fourth-order cumulant localise
(purple dashed line in \cref{fig:nlgp}C). The localisation of the CP-factors occurs
with slightly less samples than the best-performing neural network requires to
learn ($\alpha=1$). The weights of the neural networks also
localise at a sample complexity that is slightly below the sample complexity for solving the task.  The laziest network $(\alpha=100)$, i.e.\ the
one where the first-layer weights move least and which is hence closest to
random features, does not learn the task even with a training set containing
$n=10^3d$ samples when $d=20$, indicating again a large advantage of feature-learners over methods with fixed feature maps, such as random features.

\begin{figure}[t]
\centering
  \includegraphics[trim=0 580 410 0,clip,width=\linewidth]{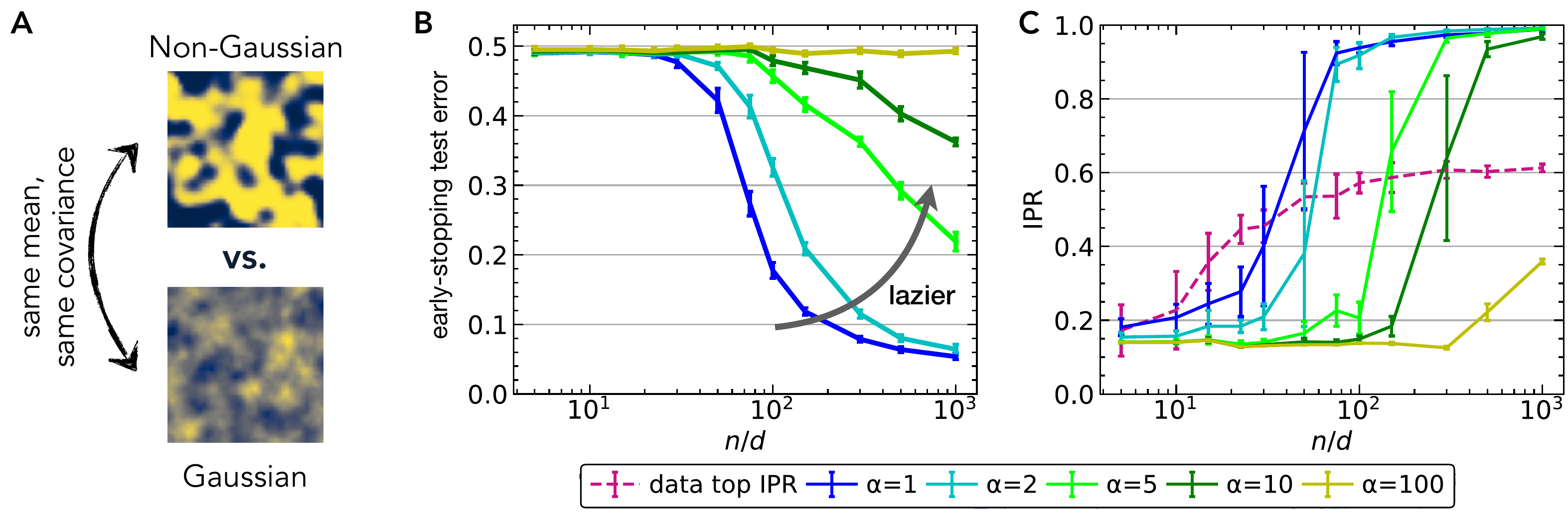}%
  \caption{\label{fig:nlgp} \textbf{A phase transition in the fourth-order
      cumulant precedes learning from the fourth cumulant.} \textbf{(A)} We
    train neural networks to discriminate inputs sampled from a simple
    non-Gaussian model for images introduced by~ \citet{ingrosso2022data} (top)
    from Gaussians with the same mean and covariance (bottom). \textbf{(B)} Test
    error of two-layer neural networks interpolating between the fully-trained
    ($\alpha=1$) and lazy regimes (large $\alpha$) -- see
    \cref{sec:experiments-nlgp}. 
    \textbf{(C)} The localisation of the leading CP-factor of the non-Gaussian
    inputs (dashed purple line) and the first-layer weights of the trained
    networks, as measured by the inverse participation ratio (IPR),
    \cref{eq:ipr}. Large IPR denotes a more localised vector
    $w$. \emph{Parameters}: $g=3$, $\xi=1, d=20, m=100$. Full
    details in \cref{app:experimental-methods}.}
    \vspace*{-1em}
\end{figure}

\section{Concluding perspectives}

Neural networks crucially rely on the higher-order correlations of their inputs
to extract statistical patterns that help them solve their tasks. Here, we have
studied the difficulty of learning from higher-order correlations in the spiked
cumulant model, where the first non-trivial information in the data set is
carried by the input cumulants of order 4 and higher. Our LR analysis of the
corresponding hypothesis test confirmed that data sampled from the spiked
cumulant model could be statistically distinguishable (in the sense that it passes the \emph{second moment method for distinguishability}) from isotropic Gaussian inputs
at linear sample complexity, while the number of samples required to strongly
distinguish the two distributions in polynomial time scales as $n \gtrsim d^2$
for the class of algorithms covered by the low-degree
conjecture~\citep{barak2019nearly, hopkins2017power, hopkins2017bayesian,
  hopkins2018statistical}, suggesting the existence of a large
statistical-to-computational gap in this problem. Our experiments with neural
networks show that they learn from HOCs efficiently in the sense that they match
the sample complexities predicted by the analysis of the hypothesis test, which
is in stark contrast to random features, which require a lot more data.
In the future, a key challenge will be extend this framework to null hypotheses that go beyond isotropic Gaussian distributions. It will be intriguing to analyse the \emph{dynamics} of neural networks
on spiked cumulant models or the non-linear Gaussian process to understand how neural networks extract
information from the higher-order cumulants of realistic data sets efficiently~\citep{bardone2024sliding}.

\section*{Acknowledgements}
We thank Zhou Fan, Yue Lu, Antoine Maillard, Alessandro Pacco, Subhabrata Sen,
Gabriele Sicuro, and Ludovic Stephan for stimulating discussions on various
aspects of this work. SG acknowledges co-funding from Next Generation EU, in the
context of the National Recovery and Resilience Plan, Investment PE1 – Project
FAIR ``Future Artificial Intelligence Research'', and from the European Union - NextGenerationEU, in the framework of the PRIN Project SELFMADE (code 2022E3WYTY – CUP G53D23000780001).

\section*{Contributions}
ES performed the numerical experiments with neural networks and random
features. LB performed the (low-degree) likelihood analysis. FG performed the
replica analysis of random features. SG designed research and advised ES and
LB. All authors contributed to writing the paper.

\printbibliography


\newpage
\appendix


\begin{center}
    \Huge
    Appendix
\end{center}
\section{Experimental details}%
\label{app:experimental-methods}

\subsection{\Cref{fig:spiked_wishart,fig:spiked_cumulant}}

For the spiked Wishart and spiked cumulant tasks, we trained  two-layer ReLU neural networks  $\phi_\theta(x) = v^\top \max\left(0, w^T x\right)$. The number of hidden neurons $m=5d$, where $d$ is the input dimension. 
We train the networks using SGD with a learning rate of 0.002 and a weight-decay of 0.002 for 50 epochs for the spiked Wishart task and for 200 epochs for the spiked cumulant task. The plots show the early-stopping test errors. The results are plotted as averages over 10 random seeds, showing the standard deviation by the errorbars.

The random features (RF) models~\citep{rahimi2008random,rahimi2009weighted} have also a width of $5 \InputDimension$. The ridge regression is performed using scikit-learn~\citep{scikit-learn} with a regularisation of 0.1.

For the spiked datasets, the spikes are from the Rademacher distribution, using a signal-to-noise ratio of 5.0 for the spiked Wishart and 10.0 for the spiked cumulant datasets.
For the overlaps between spikes and features overlaps, we plot the highest overlap amongst the incoming weights of the hidden neurons with the spike, calculated as the normalised dot product.

\paragraph{Starting from small initial overlap} Since the neural networks have a
large number of neurons ($m=5d$), some of them will have a relatively large
overlap with the spike at initialisation, as can be seen in
\cref{fig:spiked_cumulant}. To check that this relatively large overlap did not
affect our results, we repeated the same set of experiments while enforcing an
overlap of all hidden neuron weights with the spike of $\nicefrac{1}{\sqrt d}$
by explicit orthogonalisation, as discussed in
\cref{sec:experiments-wishart}. We show the results in
\cref{fig:small-initial-overlap}: while the maximum overlaps do indeed decrease
for small sample complexities, the qualitative behaviour is unchanged.

\begin{figure*}[h]
  \centering
  \includegraphics[trim=0 0 0 100,clip,width=\linewidth]{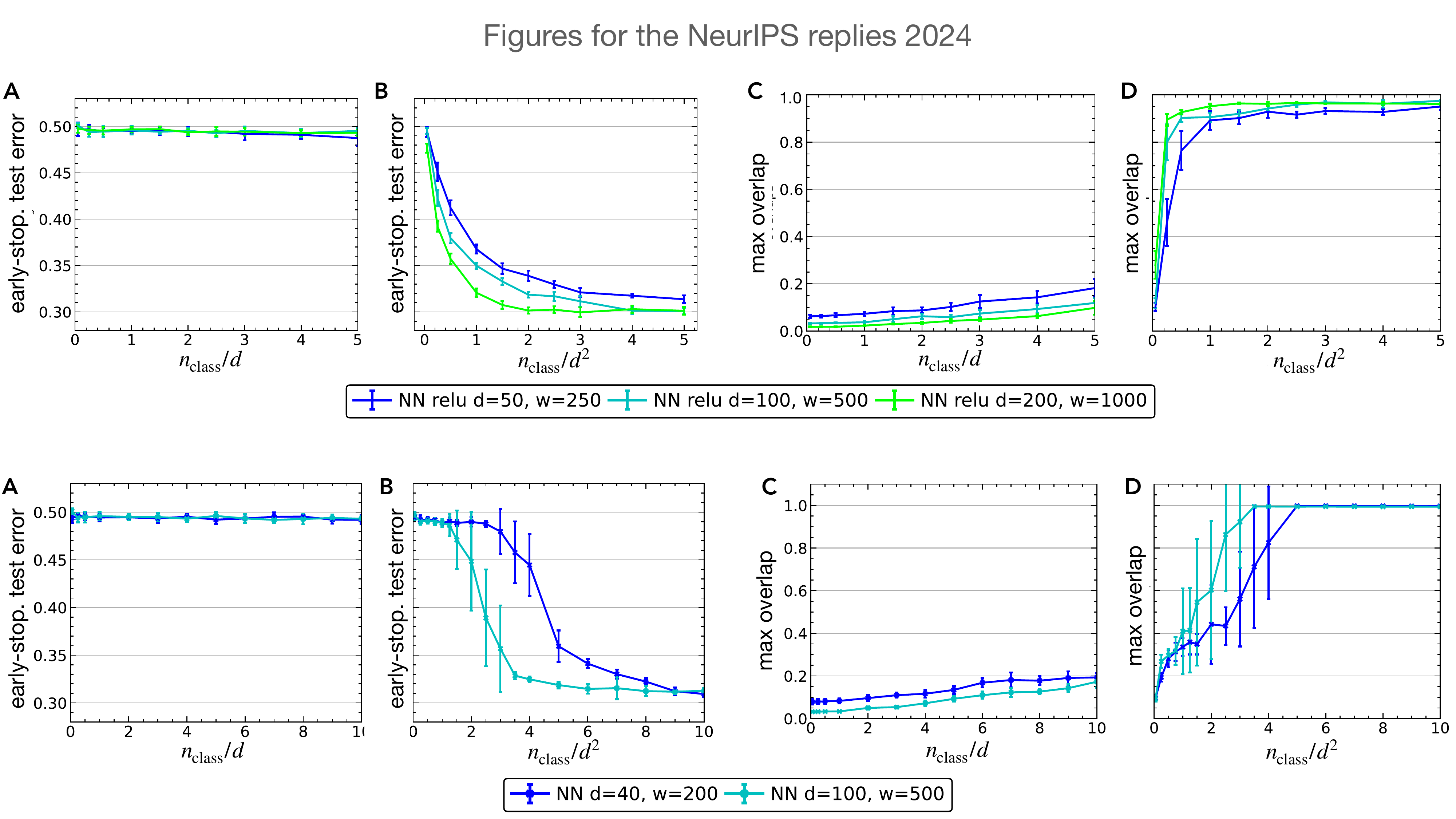}
  \caption{\label{fig:small-initial-overlap} \textbf{Learning the spiked Wishart
      and spiked cumulant task, starting from small initial overlaps.} We repeat
    the neural network experiments on the spiked Wishart (top) and spiked
    cumulant (bottom) task, see \cref{fig:spiked_wishart,fig:spiked_cumulant}, while enforcing that all hidden
    neurons have an overlap of exactly $\nicefrac{1}{\sqrt d}$ with the spikes,
    by simple explicit orthogonalisation. While the maximum overlaps do indeed
    decrease for small sample complexities, the qualitative behaviour is
    unchanged. All hyper-parameters as in \cref{fig:spiked_wishart,fig:spiked_cumulant}, respectively.}
\end{figure*}

\subsection{\Cref{fig:nlgp}}

For the NLGP--GP task, we use the $\alpha$-scaling trick of
\citet{chizat2019lazy} to interpolate between feature- and lazy-learning. We define the network function as:
\begin{equation}
 \phi_\mathrm{NN}(x; v, W,v_0, W_0) =\frac{\NTK_alpha}{K}  \left[ \sum_j^m v_j \sigma \left( \sum_i^d w_i x_i \right) -\sum_j^m v_{0,j} \sigma \left( \sum_i^d w_{0,i} x_i \right) \right]
\end{equation}
where $v_0, w_0$ are kept fixed at their initial values. The mean-squared loss is also rescaled by ($1/\NTK_alpha^2$). Changing $\NTK_alpha$ from low to high allows to interpolate between the feature- and the lazy-learning limits as the first-layer weights will move away less from their initial values.

For \cref{fig:nlgp}, the network has a width of 100 and the optimisation is run by SGD for 200 epochs with a weight-decay of $5\cdot 10^{-6}$ and learning rate of $0.5$.
The one-dimensional data vectors have a length of 20; the correlation length is 1 and the signal-to-noise ratio is set to 3. The error shown is early-stopping test error.
The localisation of the neural networks' features and the data's fourth moments is shown by the IPR measure. (We used here a lower length for the data vectors so that the calculations for fourth-order cumulants do not have high memory requirements.) For the neural networks, the highest IPR is shown amongst the incoming features of the hidden neurons at the final state of the training. For the data, the highest IPR is the highest amongst the CP-factors of the fourth cumulants of the nlgp-class using a rank-1 PARAFAC decomposition from the tensorly package~\cite{tensorly} .

\section{Mathematical details on the hypothesis testing problems}\label{app:math_details}

In this section, we provide more technical details on Hermite polynomials, LR and LDLR, and present the complete proofs of all theorems stated in the main.

\subsection{Notation}%
\label{sec:notation}
 We use the convention $0\in \N$.
 $n\in \N$ denotes the number of samples, $d$ the dimensionality of each data point $x$. The letters $k, m$ usually denote free natural parameters. Let  $[n]:=\left\{ 1,\dots,n\right\}$, 
 $\mu\in [n]$ is an index that ranges on the samples and $i\in [d]$ usually ranges on the data dimension. 
 Underlined letters $\ux,\uP,\uQ$ are used for objects that are $n\times d$ dimensional.
 In proofs the letter $C$ denotes numerical constants whose precise value may change from line to line, and any dependency is denoted as a subscript.
 $m!$ denotes factorial and $m!!$ denotes the double factorial (product of all the numbers $k\le m$ with the same parity as $m$). 
 
\paragraph{Multi-index notation} We will need multi-index notation to deal with $d$-dimensional functions and polynomials. Bold greek letters are used to denote multi-indices and their components, $\amult=\left(\amult_1,\dots,\amult_d\right)\in \N^{d}$. The following conventions are adopted:
\begin{equation}\label{eq:multi_indexproperties}
\begin{aligned}
    |\amult|:=\sum_{i=1}^d \amult_i, \quad
    \amult!&:=\prod_{i=1}^d \amult_i!, \quad
    x^\amult :=\prod_{i=1}^d x_i^{\amult_i}\\
    \partial_\amult f(x_1,\dots,x_d)&=\frac{\partial}{\partial x^\amult}f:=\left( \frac{\partial}{\partial x_1^{\amult_1}}\dots \frac{\partial}{\partial x_d^{\amult_d}}\right)f
\end{aligned}
\end{equation}

Since we will have $n$ samples of $d$-dimensional variables, we will need to consider polynomials and functions in $nd$ variables. To deal with all these variables we introduce multi-multi-indices (denoted by underlined bold Greek letters $\ammult,\bmmult,\dots$). They are $n\times d$ matrices with entries in $\N$ (i.e.\ elements of~$\N^{n\times d}$)
\begin{equation}
    \ammult:=\left(\amult^{\mu}_i\right)=\begin{pmatrix}
        \amult^{1}_1& \dots& \amult_d^{1}\\
        \vdots& \ddots& \vdots\\
        \amult^{n}_1& \dots& \amult_d^{n}
    \end{pmatrix}
\end{equation}
We denote by $\amult^{\mu}$ the rows of $\ammult$, that are $d$-dimensional multi-indices.

All the notations \eqref{eq:multi_indexproperties} generalize to multi-multi-indices in the following way:
\begin{equation}\label{eq:multi__multi_indexproperties}
\begin{aligned}
    |\ammult|:=\sum_{\mu=1}^{n} |\amult^{\mu}|&=\sum_{\mu=1}^{n}\sum_{i=1}^d \amult^\mu_i, \quad
    \ammult!:=\prod_\mu \amult^{\mu}!=\prod_\mu\prod_{i} \amult^{\mu}_i!, \\
    \ux^\ammult & :=\prod_{\mu=1}^n \left(x^{\mu}\right)^{\amult^{\mu}}=\prod_{\mu=1}^{n}\prod_{i=1}^d \left(x^{\mu}_i\right)^{\amult^{\mu}_i}\\
    \partial_\ammult f(\ux)&:=\left( \frac{\partial}{\partial (x^{1})^{\amult^1}}\dots \frac{\partial}{\partial (x^{n})^{\amult^{n}}}\right)f(\ux)
\end{aligned}
\end{equation}

\subsection{More on LR and LDLR}%
\label{sec:LDLRtheory}
\subsubsection{Statistical and computational distinguishability}

From a statistical point of view, distinguishing two sequences of
probability measures $\uP=(\P_n)_{n\in \N}$ and $\uQ=(\Q_n)_{n\in \N}$ defined on a
sequence of measure spaces $(\Omega_n,\mathcal F_n)$ means finding a
sequence of tests $f_n:\Omega_n\to \{0,1\}$, which are measurable functions that
indicate whether a given set of inputs was sampled from $\mathbb{P}$ or~$\mathbb{Q}$. We will say that the two measures are \textbf{statistically
  distinguishable} in the strong sense if there exists a statistical test $f$
for which
\begin{equation}
  \Q_n\left(f_n(x)=0\right)\underset{n\to\infty}{\longrightarrow} 1  \qand \P_n(f_n(x)=1)\underset{n\to\infty}{\longrightarrow}1.
\end{equation}
The strong version of statistical distinguishability requires that the
probability of success of the statistical test must tend to 1 as
$n\to\infty$, whereas \textbf{weak distinguishability} just requires it to be asymptotically greater than $1/2+\varepsilon$, for any $\varepsilon>0$. We will call the minimal number of samples required to achieve strong
statistical distinguishability the \textbf{statistical sample complexity} of the
problem. We can obtain a computationally bounded analogue of this definition by
restricting the complexity of the statistical test $f_n$ to functions that are
computable in a time that is polynomial in the input dimension $d$. The
\textbf{computational statistical complexity} is then the minimal number of
samples required to achieve \textbf{computational distinguishability}.

\subsubsection{Necessary conditions for distinguishability}

A necessary condition for strong distinguishability is based on the 
\emph{likelihood ratio} (LR) of probability measures, which is defined as
\begin{equation}
L_n(x):= \frac{\deriv \P_n}{\deriv \Q_n}(x)
\end{equation}
The likelihood ratio provides a necessary condition for 
strong distinguishability of $\P$ and $\Q$ via the so-called second moment method:

\begin{proposition}[Second Moment Method for Distinguishability]%
  \label{prop:SMMD}
  Suppose that $\P_n$ is absolutely continuous with respect to $\Q_n$, and let
  $L_n$ be the corresponding LR. A necessary condition for strong
  distinguishability of~ $\underline{\P}$ from~$\underline \Q$ is
  \begin{equation}
    \norm{L_n}^2 :=\underset{x\sim
      \Q_n}{\E}[L_n(x)^2] \underset{n\to \infty }{\longrightarrow}+ \infty.
  \end{equation}
  where $\norm{\cdot}$ is the norm with respect to the Hilbert space  
\begin{equation}\label{eq:l2space}
L^2(\Omega_n,\Q_n)=\left\{f:\Omega_n\to \R \ | \ \E_{\Q_n}[f^2(x)]<\infty  \right\}.
\end{equation}
\end{proposition}
\begin{proof}
    The proof is immediate:  if $(||L_n||)_n$ were bounded, by Cauchy-Schwartz $\Q(A_n)\to 0$ would imply $\P(A_n)\to 0$. But by the definition of strong asymptotic distinguishability there exists a sequence of events $(A_n)_n$ such that $\P(A_n)\to 1$ and $\Q(A_n)\to 0$; hence $(||L_n||)_n$ must diverge.
\end{proof}
The second moment method has been used to derive statistical thresholds for
various high-dimensional inference problems~\citep{montanari2015limitation,
  banks2016information, perry2016optimality, kunisky2019notes}. Note that proposition~\ref{prop:SMMD} only states a necessary condition; while it is possible to
construct counterexamples, they are usually based on rather artificial
constructions (like the following example 1 from \citep{kunisky2019notes}), and the second
moment method is considered a good proxy for strong statistical
distinguishability.

\begin{example} Let $\underline \P$ and $\underline \Q$ be two strongly distinguishable sequences of probability measures on $(\mathcal Y_n,\mathcal F_n)_n$. Define $\P_n'$ on $\mathcal F_n$ as a measure that with probability $1/2$ follows $\P_n$ and with probability $1/2$ follows $\Q_n$. Letting $L'_n$ be the LR of $\P_n'$ with respect to $\Q_n$. We have that $||L_n'||\to \infty$, but $\underline \P'$ and $\Q$ are not strongly distinguishable. 
\end{example}
\begin{proof}The density of $\P'_n$ is:
\begin{equation}p'_n(y)=\frac{p_n(y)+q_n(y)}2
\end{equation}
hence 
\begin{equation}L_n'=\frac 12\left(1+L_n\right)
\end{equation}
which clearly has the same asymptotic behaviour as $L_n$.

On the other hand, it cannot be that $\underline \P'$ and $\underline \Q'$ are strongly distinguishable since for any event $A$, $\P'_n(A)\ge \frac 12 \Q_n(A)$.  
\end{proof}

\subsubsection{LDLR analysis for spiked Wishart model}
\label{sec:wishart}
In this subsection we show the application of LDLR method to the spiked Wishart model. We obtain the correct BBP threshold for the signal-to-noise ratio even when we restrict ourselves to the class of polynomials of constant degree:
\begin{theorem}[LDLR for spiked Wishart model]%
  \label{thm:ldlr-wishart}
  Suppose the prior on $u$ belongs to the following cases:
  \begin{itemize}
  \item $(u_i)_{i=1,\dots,d}$ are i.i.d.\ and symmetric Rademacher random
    variables
  \item $(u_i)_{i=1,\dots,d}$ are i.i.d.\ and $u_i\sim \mathcal N(0,1)$
  \end{itemize}
  Let $D\in\N$ and $d,n\to \infty$, with fixed ratio $\gamma:=d/n$, then
  \begin{equation} \label{eq:LDLRwishart}
    \lim_{d,n\to \infty} \|L^{\le D}\|=\sum_{k=0}^{D} \frac{(2k-1)!!}{(2k)!!} \frac{\beta^{2k}}{\gamma^k},
  \end{equation}
  which, as $D$ increases, stays bounded for $\beta <\beta_c:=\sqrt \gamma=\sqrt{\nicefrac{d}{n}}$ and diverges for
  $\beta>\beta_c$.
\end{theorem}
The distinguishability threshold that we have recovered here is of course the
famous BBP phase transition in the spiked Wishart model~\citep{baik2005phase}:
if $\beta < \beta_c = \sqrt{\nicefrac{d}{n}}$, the low-degree likelihood ratio
stays bounded and indeed a set of inputs drawn from \ref{eq:wishart} is
statistically indistinguishable from a set of inputs drawn from
$\mathcal{N}(0, \id)$. For $\beta > \beta_c$ instead, there is a phase
transition for the largest eigenvalue of the empirical covariance of the inputs
in the second class, which can be used to differentiate the two classes, and the
LDLR diverges.

As mentioned in the main text, a thorough LDLR analysis of a more general model
that encompasses the spiked Wishart model was performed by
\citet{bandeira2020computational}. Their theorem 3.2 is a more general version
of our theorem \ref{thm:ldlr-wishart}, that generalizes it in two directions:
it works for a class of subgaussian priors on $u$ that satisfy some
concentration property, and also allows for negative SNR (the requirement is
$\beta>-1$). For completeness, here we show that this result can be obtained as a straightforward
application of a Gaussian additive model.

\begin{proof}
\label{app:ldlr-wishart-proof}

We note that the problem belongs the class of \emph{Additive Gaussian Noise
  models} which is studied in depth by \citet{kunisky2019notes}. In those models
the two hypotheses have to be expressed as an additive perturbation of white noise:
\begin{itemize}
\item $\P_n$: $\ux_n=y_n+z_n$,
\item $\Q_n$: $\ux_n=z_n$,
\end{itemize}
The spiked Wishart model that we consider belongs to this class. It can be seen by defining
$\R^{nd}\ni y_n=\left(\sqrt {\frac{\beta}d}g^1 u, \dots,\sqrt{\frac{\beta}d}g^n u
\right)^{\top}$.  So we can apply theorem 2.6 from \citet{kunisky2019notes}, that computes the norm of the LDLR by
using two independent replicas of the variable $y$. Denoting the two replicas by $\hat y$ and $\tilde y$, we get
    \begin{equation} \begin{aligned}
        ||L_n^{\le D}||^2&= \E\left[\sum_{m=0}^D \frac{1}{m!}(\hat y \cdot \tilde y)^m\right]\\
        & =\E\left[\sum_{m=0}^D \frac{\beta^m}{m!d^m}\left(\sum_{\mu=1}^n \hat g^\mu \tilde g^\mu \hat u\cdot \tilde u\right)^m\right]\\&=\sum_{m=0}^D \frac{\beta^m}{m!d^m}
        \E\left[\left(\hat u\cdot \tilde u\right)^m\left(\sum_{\mu=1}^n \hat g^\mu \tilde g^\mu \right)^m\right]\\&= \sum_{m=0}^D \frac{\beta^m}{m!d^m} \E\left[\left(\sum_{\mu=1}^n \hat g^\mu \tilde g^\mu \right)^m\right]
       \E\left[\left(\hat u\cdot \tilde u\right)^m\right] 
    \end{aligned}\end{equation} 
    Note now that $\sum_{\mu=1}^n \hat g^\mu \tilde g^\mu $ has the same distribution as $-\sum_{\mu=1}^n \hat g^\mu \tilde g^\mu$, so its distribution is even and all the odd moments are 0. This means that we can reduce to the case $m=2k$, so we need to study:
    \begin{equation}
        ||L_n^{\le D}||^2=\sum_{k=0}^{\lfloor D/2\rfloor} \frac{\beta^{2k}}{(2k)!d^{2k}} \underbrace{\E\left[\left(\sum_{\mu=1}^n \hat g^\mu \tilde g^\mu \right)^{2k}\right]}_{T_1}
        \underbrace{\E\left[\left(\hat u\cdot \tilde u\right)^{2k}\right]}_{T_2}
    \end{equation}
    Let us consider the term $T_1$. Call $Y^\mu:=\hat g^\mu \tilde g^\mu$. The distribution of each of the $Y^\mu$ is not Gaussian, but we have that $\E[Y^\mu]=0$ and Var$(Y^\mu)=1$. So by the central limit theorem $S_n:=\frac{1}{\sqrt n}\sum_{\mu=1}^n Y^\mu$ converges to a standard normal in distribution, as $n\to \infty$. Note that the cumulants of $S_n$ can be computed thanks to linearity and additivity of cumulants and are $\kappa^{(S_n)}_{2k}=n^{1-k}\kappa^{Y}_{2k}$. So they all go to 0 except from the variance. Since the moments can be written as a function of the cumulants up to that order, it follows that $\lim_n\E[S^{2k}_n]$ will be the $2k$-th moment of the standard normal distribution, which means that we have the following:
    \begin{equation}
            \lim_{n\to +\infty}  \E\left[\left(\frac{1}{\sqrt n}\sum_{\mu=1}^n \hat g^\mu \tilde g^\mu \right)^{2k}\right]=(2k-1)!!
    \end{equation}
We turn now to T2 and do the same reasoning. Define $v_i:=\hat u_i\tilde u_i$, both in Rademacher and in Gaussian prior case, we have that the $(v_i)_{i=1,\dots, d}$ is an independent family of random variables that have 0 mean and variance equal to 1. So we can again apply the central limit theorem to get that:
\begin{equation}
            \lim_{d\to +\infty} \E\left[\left(\frac{\hat u\cdot \tilde u}{\sqrt d}\right)^{2k}\right]=(2k-1)!!
    \end{equation}
Taking the limit for $n,d\to +\infty$ with the constraint $\gamma=d/n$, we have that:
\begin{equation}
\begin{aligned}
     \lim_{n,d \to \infty}||L_n^{\le D}||^2&=\lim_{n,d \to \infty} \sum_{k=0}^{\lfloor D/2\rfloor} \frac{\beta^{2k}n^k}{(2k)!d^{k}} \E\left[\left(\frac{1}{\sqrt n}\sum_{\mu=1}^n \hat g^\mu \tilde g^\mu \right)^{2k}\right]
        \E\left[\left(\frac{\hat u\cdot \tilde u}{\sqrt d}\right)^{2k}\right] \\
        &= \sum_{k=0}^{D} \frac{\left(\beta^{k}(2k-1)!!\right)^2}{(2k)!\gamma^k}= \sum_{k=0}^{D} \frac{(2k-1)!!}{(2k)!!} \frac{\beta^{2k}}{\gamma^k}
        \end{aligned}    
\end{equation}
which is what we wanted to prove.

As a final note we remark that the basic ideas of the arguments in \citep{bandeira2020computational} coincide with what exposed above. However, the increased generality of the statement requires the use of the abstract theory of Umbral calculus to generalize the notion of Hermite polynomials to negative SNR cases, as well as more technical work to achieve the bounds on the LDLR projections. 
\end{proof}
\subsection{Hermite Polynomials}%
\label{subsec:hermite}

We recall here the definitions and key properties of the Hermite polynomials.
\begin{definition}The Hermite polynomial of degree $m$ is defined as
\begin{equation}h_m(x):=(-1)^m e^{\frac {x^2}2} \frac{\deriv^m}{\deriv x^m}\left(e^{-\frac {x^2}2}\right)
\end{equation}
\end{definition}
Here is a list of the first 5 Hermite polynomials:
\begin{equation} \begin{aligned}
    h_0(x)&=1\\
    h_1(x)&=x\\
    h_2(x)&=x^2-1\\
    h_3(x)&=x^3-3x\\
    h_4(x)&=x^4-6x^2+3
\end{aligned}\end{equation} 
The Hermite polynomials enjoy the following properties that we will use in the
subsequent proofs (for details see section
5.4 in \citet{mccullagh2018tensor}, \citet{szego1939orthogonal} and
\citet{abramowitz1964handbook}):
 \begin{itemize}
     \item they are orthogonal with respect to the $L^2$ product weighted with the density of the Normal distribution:
       \begin{equation}
         \frac 1{\sqrt{2\pi}} \int_{-\infty}^\infty h_n(x) h_m(x) e^{-\frac {x^2}2}\deriv x=n!\delta_{m,n};
       \end{equation} 
     \item $h_m$ is a monic polynomial of degree $m$, hence $(h_m)_{m\in\{1,\dots,N\}}$ generates the space of polynomials of degree $\le N$;
     \item the previous two properties imply that the family of Hermite polynomials is an orthogonal basis for the Hilbert space $L^2(\R,\Q)$ where $\Q$ is the normal distribution;
     \item they enjoy the following recurring relationship
     \begin{equation} \label{eq:HermPol_recur}
         h_{m+1}(x)=x h_m(x)-h_{m}'(x),
     \end{equation}
     which can also be expressed as a relationship between coefficients as
     follows. If $h_m(x)=\sum_{k=0}^m a_{m,k} x^k$, then
     \begin{equation} \label{eq:HermPol_coef_recur}
         a_{m+1,k}=\begin{cases}
             -a_{m,1} & k=0,\\
             a_{m,k-1}-(k+1)a_{m,k+1}& k>0.
         \end{cases}
     \end{equation}
     They also satisfy identities of binomial type, like:
     \begin{gather}%
       \label{eq:HermPolBinomId_sum}
       h_m(x+y)=\sum_{k=0}^m\binom{m}{k}x^{m-k}h_k(y)\\
       \label{eq:HermPolBinomId_scalmult}
       h_m(\gamma x)=\sum_{j=0}^{\lfloor m/2\rfloor} \gamma^{m-2j}(\gamma^2-1)^j\binom{m}{2j}  (2j-1)!! h_{m-2j}\left(x\right)
     \end{gather}
   \end{itemize}

\paragraph{Multivariate case}
In the multivariate $m$-dimensional case we can consider Hermite tensors $(H_\amult)_{\amult\in \N^m}$ defined as:
 \begin{equation}H_\alpha(x_1,\dots,x_m)=\prod_{i=1}^m h_{\alpha_i}(x_i)
 \end{equation}
 all the properties of the one-dimensional Hermite polynomials generalize to this case, in particular they form an orthogonal basis of $L^2(\R^m,\Q)$ where $\Q$ is a multivariate normal distribution. If $\langle \cdot,\cdot\rangle$ is the inner product of that Hilbert space, we have that:
 \begin{equation}\langle H_\amult,H_\bmult\rangle=\amult!\delta_{\amult,\bmult}
 \end{equation}
Of course all this is valid in the case $m=nd$, where we can also use multi-multi-indices to get the following identity:
\begin{equation}H_\ammult(x_1^1,\dots,x_d^n)=\prod_{\mu=1}^n H_{\amult^\mu}(x^\mu)=\prod_{\mu=1}^n\prod_{i=1}^d h_{\amult^\mu_i}(x_i^\mu)
\end{equation}

We are now ready to see the proof of \ref{eq:normLDLRgen}.
\begin{lemma} \label{lemma:normLDLRgen} We consider  a hypothesis testing problem in $\R^d$ where the null hypothesis $\Q$ is the multivariate Gaussian distribution $\mathcal N(0,\mathbbm 1_{d\times d})$, while the alternative hypothesis $\P$ is absolutely continuous with respect to it.
Then:
\begin{equation}
L^{\le D}=\underset{|\amult|\le D}{\sum}\frac {\langle L,H_{\amult}\rangle H_\amult}{\amult!}=\underset{|\amult|\le D}{\sum}\frac {\underset{x \sim \P}{\E}[H_{\amult}(x)]H_\amult}{\amult!}
\end{equation}
Which implies
\begin{equation}
||L^{\le D}||^2=\underset{|\amult|\le D}{\sum}\frac {\langle L,H_{\amult}\rangle^2}{\amult!}=\underset{|\amult|\le D}{\sum}\frac {\underset{x \sim \P}{\E}[H_{\amult}(x)]^2}{\amult!}
\end{equation}
\end{lemma}
\begin{proof}
    First note that 
    \begin{equation}
    \langle L,H_{\amult}\rangle= \underset{x \sim \P}{\E}[H_{\amult}(x)]
    \end{equation}
    due to the definition of likelihood ratio and a change of variable in the expectation.
    Then we can just use the fact that $(H_\amult)_\amult \in \N^d$ are an orthogonal basis for $L^2(\R^d, \Q)$, and if we consider the Hermite polynomials  up to degree $D$, they are also a basis of the space of polynomials in which we want to project $L$ to get $L^{\le D}$. Hence the formulas follow by just computing the projection using this basis.
\end{proof}
Note that we set the lemma in $\R^d$, but of course it holds also in $\R^{nd}$ just switching to multi-multi-index notation.

\subsubsection{Hermite coefficients}
Lemma \ref{lemma:normLDLRgen} translates the problem of computing the norm of the LDLR to the problem of computing the projections $\langle L,H_\amult\rangle$. Note that this quantity is equal to $\E_\P[H_\amult(x)]$, which we will call $\amult$\emph{-th Hermite coefficient} of the distribution $\P$.

The following lemma from \citep{bandeira2020computational} provides a version of the integration by parts technique that is tailored for Hermite polynomials.
\begin{lemma} \label{lemma:IntByParts}
    Let $f : \R^d \to \R^d$ be a function that is continuously differentiable $k$ times. Assume that $f$ and all of its partial derivatives up to order $k$ are bounded by $O(\exp(|y|^\lambda))$ for some $\lambda \in (0, 2)$, then for any $\amult \in \N^d$ such that $|\amult|\le k$
\begin{equation} \label{eq:IntByParts}
\langle f,H_\amult\rangle=\underset{y\sim \mathcal N(0,\id)}\E\left[H_\amult(y)f(y)\right] =\underset{y\sim \mathcal N(0,\id)}\E\left[\partial_\amult f (y)\right]
\end{equation}
\end{lemma}
\begin{proof}
    \ref{eq:IntByParts} can be proved by doing induction on $k$ using \ref{eq:HermPol_recur}, see \citep{bandeira2020computational} for details.
\end{proof}
\subsubsection{Links with cumulant theory}
The cumulants $(\kappa_\amult)_{\amult\in \N^d}$ of a random variable $x\in \R^d$, can be defined as the Taylor coefficients of the expansion in $\xi$ of the \emph{cumulant generating function}:
\begin{equation}
   K_x(\xi):= \log\left(\E\left[e^{\xi \cdot x}\right]\right)
\end{equation}
Order-one cumulant is the mean, order-two is the variance, and higher-order cumulants cumulants encode more complex correlations among the variables. The Gaussian distribution is the only non constant distribution to have a polynomial cumulant generating function (as proved in theorem 7.3.5 in \citep{lukacs1972survey}), if $z\sim \mathcal N(\mu,\Sigma)$, then:
\[
K_z(\xi)=\mu \xi+\frac 12\xi^\top \Sigma \xi .
\]
Hence cumulants with order higher than three can also be seen as a measure of how much a distribution deviates from Gaussianity (i.e.\ how much it deviates from its best Gaussian approximation).

On this point the similarity with Hermite coefficients is evident. Indeed, up to
order-five, on whitened distributions, cumulants and Hermite coefficients
coincide. But form sixth-order onward, they start diverging. They are still
linked by deterministic relationship, but its combinatorial complexity increases
swiftly and it is not easy to translate formulas involving Hermite coefficients
into cumulants and vice versa. For this reason, our low-degree analysis of the
likelihood ratio leading to \ref{thm:ldlr-cumulant} keeps the formalism that
naturally arises from the computations, which is based on Hermite
coefficients. A detailed discussion of the relation between Hermite polynomials
and cumulants in the context of asymptotic expansions of distributions like the
Gram--Charlier and Edgeworth series can be found in chapter 5 of \citet{mccullagh2018tensor}.

\subsection{Details on the spiked cumulant model}%
 \label{subsec:LDLRdetails}

Here we will expand on the mathematical details of the spiked cumulant model.

\subsubsection{Prior distribution on the spike}
For the prior distribution on $u$, $\mathcal P$, its role is analogous to the spiked Wishart model, so all the choices commonly used for that model can be considered applicable to this case. Namely symmetric distributions so that
 $\big |\big|\frac u{\sqrt d}\big|\big| \approx 1$ as $d\to \infty$. 
In the following we will make the computations assuming $u_i$ i.i.d.\ and with Rademacher prior:
\begin{equation}
u_i \sim \text{Rademacher}(1/2)
\end{equation}
It helps to have i.i.d.\ components and  constant norm $\big |\big|\frac u{\sqrt d}\big|\big| \equiv 1$. However all the results should hold, with more involved computations, also with the following priors:
\begin{itemize}
\item $(u_i)_{i=1,\dots,d}$ are i.i.d.\ and $u_i\sim \mathcal N(0,1)$
\item $u\in \text{Unif}(\partial B(0,\sqrt d)$.
\end{itemize}

\subsubsection{Distribution of the non-Gaussianity}
As detailed in \ref{assumptions:g}, we need the non-Gaussianity to satisfy specific conditions. Some of the requirements are fundamental and cannot be avoided, whereas others could likely be removed or modified with only technical repercussion that do not change the essence of the results. 

The most vital assumptions are the first and the last: $\E[g]\ne 0$ would introduce signal in the  first cumulant, changing completely the model. It is important to control the tails of the distribution with
\begin{equation}
  \label{eq:assumpTails}
  \E[\exp(g^2/2)]<+\infty,
\end{equation} a fat-tailed $g$ would make the LR-LDLR technique pointless due to the fact that $L\notin \mathcal{L}^2(\R^{d},\Q)$ and $||L||=\infty$ for any $n,d$. For example it is not possible to use Laplace distribution for $g$.

On the opposite side, the least important assumptions are that Var$(g)=1$ and
\cref{eq:assumpHermCoef}; removing the first would just change the formula for
whitening matrix $S$, while \cref{eq:assumpHermCoef} is a very weak requirement
and it is needed just to estimate easily the Hermite coefficients' contribution
and reach \ref{eq:asymp_upper_bound}, it may be even possible to remove it and
try to derive a similar estimate from the assumption on the tails \ref{eq:assumpTails}.

Finally, the requirement of symmetry of the distribution has been chosen arbitrarily thinking about applications: the idea is that the magnitude of the fourth-order cumulant, the kurtosis, is sometimes used as a test for non-Gaussianity, so it is interesting to isolate the contribution given by the kurtosis, and higher-order, even degree, cumulants, by cancelling the contribution of odd-order cumulants.
However, it would be interesting to extend the model in the case of a centered distribution with non-zero third-order cumulant. It is likely that the same techniques can be applied, but the final thresholds on $\theta$ may differ.

The following lemma gives a criterion of admissibility that ensures Radem$(1/2)$
and Unif$(-\sqrt 3,\sqrt 3)$ (together with many other compactly supported
distributions) satisfy \ref{assumptions:g}.
\begin{lemma}
    Suppose $p_g$ is a probability distribution, compactly supported in $[-\Lambda,\Lambda]$, $\Lambda\ge 1$, then
    \[
    \underset{g\sim p_g }{\E}[h_m(g)]\le \Lambda^m m!
    \]
\end{lemma}
\begin{proof}
    Use the notation $h_m(x)=\sum_{k=0}^m a_{m,k} x^k$ and $S_m:=\sum_{k=0}^m |a_{m,k}|$. Then we have that
    \begin{equation}
        \begin{aligned}
            \underset{g\sim p_g }{\E}[h_m(g)]&= \underset{g\sim p_g }{\E}\left[\sum_{k=0}^m a_{m,k} g^k\right]\\
            & \le  \underset{g\sim p_g }{\E}\left[\sum_{k=0}^m |a_{m,k}| |g|^k\right]\\
            & \le \Lambda^m S_m\\
        \end{aligned}
    \end{equation}
    So we just need to prove that $S_m\le m!$, which can be done by induction using \ref{eq:HermPol_coef_recur}. Suppose it true for $m$, we prove it for $m+1$. By \ref{eq:HermPol_coef_recur} we have that:
    \begin{equation} \label{eq:HermPol_coef_recur_abs_val}
         |a_{m+1,k}|\le\begin{cases}
             |a_{m,1}|& k=0\\
             |a_{m,k-1}|+(k+1)|a_{m,k+1}|& k>0
         \end{cases}
     \end{equation}
     Summing on both sides (and using $a_{m,k}=0$ when $k>m$), we get: 
     \begin{equation}
         \begin{aligned}
             S_{m+1}&\le |a_{m,1}|+\sum_{k=1}^{m+1}|a_{m,k-1}|+(k+1)|a_{m,k+1}|\\
             S_{m+1}&\le S_m+\sum_{j=0}^{m}j|a_{m,j}|\\
             S_{m+1}&\le S_m+m\sum_{j=0}^{m}|a_{m,j}|\\
             S_{m+1}&\le (m+1)S_m             
         \end{aligned}
     \end{equation}
     Hence by application of the inductive hypothesis we get $S_{m+1}\le (m+1)!$, completing the proof.
\end{proof}

\subsubsection{Computing the whitening matrix}%
\label{app:whitening}
In this paragraph all the expectations are made assuming $u$ fixed and it will be best to work with its normalized version $\bar u=u/\sqrt{d}$.  Note that $\mathbb E[x]=0$ and we want also that
\begin{equation} \begin{aligned}\mathbbm 1_{d\times d}&=\E[xx^{\top}]=\E[S\left(\sqrt {\beta} g \bar u+z\right) \left(\sqrt \beta g \bar u^{\top}+z^{\top}\right)S^{\top}]\\
&=S\left(\id+\beta \bar u\bar u^{\top}\right)S^{\top}
\end{aligned}\end{equation} 
Hence we need 
\begin{equation}
S^2=SS^{T}=\left(\id+\beta  \bar u\bar u^{\top} \right)^{-1}=\id-\frac{\beta}{1+\beta}\bar u\bar u^{\top}
\end{equation}
So we look for $\gamma$ such that:
\begin{equation}
(\id+\gamma \bar u\bar u^{\top})^2=\id-\frac{\beta}{1+\beta} \bar u\bar u^{\top}
\end{equation}
By solving the second-order equation we get
\begin{equation}
\gamma_{\pm}= -\left(1\pm\frac1{\sqrt{1+\beta }}\right)
\end{equation}
We choose the solution with $-$, so that $S$ is positive definite:
\begin{equation}   
S=\id-\left(1-\frac1{\sqrt{1+\beta }}\right)\bar u\bar u^{\top} =\id - \frac{\beta}{1 + \beta + \sqrt{1 + \beta}} \frac{u u^\top}{d}.
\end{equation}
Hence we can compute also the explicit expression for $x$:
\begin{align}
\notag x&=z-\left(1-\frac1{\sqrt{1+\beta}}\right)\bar u^{\top}z\bar u+\sqrt{\frac{\beta}{1+\beta}}g\bar u\\
\notag x&=\underbrace{z-\bar u^{\top}z\bar u}_{z_{\perp u}} +\left(\sqrt{\frac1{1+\beta}}\bar u^{\top}z+\underbrace{\sqrt{\frac{\beta}{1+\beta }}}_{\eta}g\right)\bar u\\
&\label{eq:xexpl}={z_{\perp u}} +\left(\sqrt{1-\eta^2}\bar u^{\top} z+\eta g\right)\bar u
\end{align}
So $x$ is standard Gaussian in the directions orthogonal to $u$, whereas in $u$ direction, it is a weighted sum between $g$ and $z$. Note that it is a quadratic interpolation: the sum of the square of the weights is 1.

\subsection{Details of the LR analysis for spiked cumulant model}%
\label{subsec:LRanalysisappendix}

\subsubsection{Proof sketch for theorem ~\ref{thm:LRspiked_cumulant}}
\label{sec:cumulant-lr-sketch}

Since the samples are independent, the total LR factorises as 
\begin{equation}
  L(\underline y)=\underset{u}\E\left[\prod_{\mu=1}^{n}l(y^{\mu}|u)\right],
\end{equation}
where the sample-wise likelihood ratio is
\begin{equation} \label{eq:integraloverlambda}
  l(y|u)=
  \frac{p_x(y|u)}{p_z(y)}=\underset{g\sim p_g}{\mathbb E}\left[\sqrt{1+\beta}
    \exp\left(-\frac{1+\beta}{2}\left(g-\sqrt{\frac \beta {(1+\beta)d}}y\cdot
        u\right)^2+\frac{g^{2}}{2}\right)\right].
\end{equation}
To compute the norm of the LR, we consider two independent replicas of the spike, $u$ and $v$, to get
\begin{equation}
\begin{aligned}
    \|L_{n,d}\|^{2}&=\underset{\underline y\sim \uQ}\E \left [ \underset{u}\E\left[\prod_{\mu=1}^{n}l(y^{\mu}|u)\right]\underset{v}\E\left[\prod_{\mu=1}^{n}l(y^{\mu}|v)\right]\right].
\end{aligned}
\end{equation}
The hard part now is to simplify the high-dimensional integrals in this expression, it can be done because the integrand is almost completely symmetric, and the only asymmetries lie on the subspace spanned by $u$ and $v$.
\[
 \|L_{n,d}\|^{2}=\underset{u,v}\E\left[ \underset{g_u,g_v}{\mathbb E}\left[\frac{1+\beta}{\sqrt{(1+\beta)^2-\beta^2\left( \frac{u\cdot v}d\right)^2}}e ^{-\frac{(1+\beta)\left((1+\beta)(g_u^2+g_v^2)-2\beta(g_ug_v)\left( \frac{u\cdot v}d\right)\right)}{2(1+\beta)^2-2\beta^2\left( \frac{u\cdot v}d\right)^2}+\frac{g_u^2+g_v^2}2}\right]^n\right]
\]
Note that the integrand depends on $u$ and $v$ only trough $\frac{u\cdot v}d=:\lambda$. So, using that the prior on $u$, $v$ is i.i.d.\ Rademacher,  the outer expectation can be transformed in a one dimensional expectation over $\lambda$, leading to \eqref{eq:LRnormf}.

\subsubsection{Proof of theorem~\ref{thm:LRspiked_cumulant}}%
\label{sec:cumulant-lr-details}

$\tilde x^\mu =\sqrt{\frac \beta d}g^\mu u+z^{\mu}$, to find the marginal density of $x^\mu$ we integrate over the possible values of $g$
\begin{equation}\label{eq:ptildex}
   p_{\tilde x}(y|u)= \P\left(\tilde x^{\mu}\in \deriv y|u\right)=\underset{g}{\mathbb E}\left[p_z\left(y-\sqrt{\frac \beta d}g u\right )\right]
\end{equation}
where $p_z$ is the density of a standard normal $d$-dimensional variable $z\sim \mathcal N(0,\id_d)$.
But we are interested in the density of the whitened variable $x$, $p_x(\cdot|u)$, which can be seen as the push forward of the density of $\tilde x$ with respect to the linear transformation $S$. So
\begin{equation}
    p_x(y|u)=p_{\tilde x}(S^{-1}y|u)|\det S^{-1}|.
\end{equation}
It is easy to see from \ref{eq:whitening} that $|\det S^{-1}|=\sqrt{1+\beta}$, so we can plug it into \ref{eq:ptildex} and after expanding the computations we get to
\begin{equation} \label{eq:brute-force_estimatorappendix}
    p_x(y|u)=p_z(y)\underset{g}{\mathbb E}\left[\sqrt{1+\beta} \exp\left(-\frac{1+\beta}{2}\left(g-\sqrt{\frac \beta {(1+\beta)d}}y\cdot u\right)^2+\frac{g^{2}}{2}\right)\right].
\end{equation}
So we have found the likelihood ratio for a single sample, conditioned on the spike:
\begin{equation}
    l(y|u)=  \frac{p_x(y|u)}{p_z(y)}=\underset{g}{\mathbb E}\left[\sqrt{1+\beta} \exp\left(-\frac{1+\beta}{2}\left(g-\sqrt{\frac \beta {(1+\beta)d}}y\cdot u\right)^2+\frac{g^{2}}{2}\right)\right].
\end{equation}

Note that conditioning on $u$ the samples are independent, so we have the following formula:
\begin{equation}
L(\underline y)=\underset{u}\E\left[\prod_{\mu=1}^{n}l(y^{\mu}|u)\right].
\end{equation}
So to compute the norm we consider two independent replicas of the spike, $u$ and $v$, then we switch the order of integration to get
\begin{equation}
\begin{aligned}\label{eq:LRnormexpl}
    ||L_{n,d}||^{2}&=\underset{\underline y\sim \uQ}\E \left [ \underset{u}\E\left[\prod_{\mu=1}^{n}l(y^{\mu}|u)\right]\underset{v}\E\left[\prod_{\mu=1}^{n}l(y^{\mu}|v)\right]\right]\\
    &=\underset{u,v}\E\left[\underset{\underline y\sim \uQ}\E \left[\prod_{\mu=1}^{n}l(y^{\mu}|u) l(y^{\mu}|v)\right]\right]\\
    &=\underset{u,v}\E\left[\underset{ y\sim \Q}\E \left[l(y|u) l(y|v)\right]^n\right] \\
   &= \underset{u,v}\E \Bigg[\underset{ y\sim \Q}\E \Bigg{(}\underset{g_u}{\mathbb E}\left[\sqrt{1+\beta} \exp\left(-\frac{1+\beta}{2}\left(g_u-\sqrt{\frac \beta {(1+\beta)d}}y\cdot u\right)^2+\frac{g_u^{2}}{2}\right)\right]\\  & \qquad \qquad \quad \cdot \underset{g_v}{\mathbb E}\left[\sqrt{1+\beta} \exp\left(-\frac{1+\beta}{2}\left(g_v-\sqrt{\frac \beta {(1+\beta)d}}y\cdot v\right)^2+\frac{g_v^{2}}{2}\right)\right]\Bigg{)}^n\Bigg{]}.
\end{aligned}
\end{equation}
Switching the integral over $y$ inside the new integrals over $g_u$ and $g_v$ we can isolate the following integral over $y$:
\begin{equation}
        I:= \underset{ y\sim \Q}\E \left[ \exp\left(-\frac{1+\beta}{2}\left(g_u-\sqrt{\frac \beta {(1+\beta)d}}y\cdot u\right)^2-\frac{1+\beta}{2}\left(g_v-\sqrt{\frac \beta {(1+\beta)d}}y\cdot v\right)^2\right)\right].
\end{equation}
It can be computed by noting the subspace orthogonal to  $\left\{u,v\right\}$, we just have the integral of a standard normal, and the remaining 2-dimensional integral can be computed explicitly. Since the Rademacher prior implies that $||u||=||v||=\sqrt d$, the result depends only on their overlap $\lambda$ (i.e.\ $\lambda=\frac{u\cdot v}d$), leading to:
\begin{equation}
  I=\frac{1}{\sqrt{(1+\beta)^2-\beta^2\lambda^2}}
  \exp\left(-\frac{1+\beta}{2(1+\beta)^2-2\beta^2\lambda^2}\left((1+\beta)(g_u^2+g_v^2)-2\beta(g_ug_v)\lambda\right)\right)
\end{equation}
If we plug this formula inside \ref{eq:LRnormexpl} and rearrange the terms we
get an expression that can be written in terms of the density of two
bi-dimensional centered Gaussians
\begin{equation} 
\begin{aligned}\label{eq:LRnorm}
   ||L_{n,d}||^{2}  &=\underset{\lambda}\E\left[ \underset{g_u,g_v}{\mathbb E}\left[\frac{1+\beta}{\sqrt{(1+\beta)^2-\beta^2\lambda^2}}e ^{-\frac{(1+\beta)\left((1+\beta)(g_u^2+g_v^2)-2\beta(g_ug_v)\lambda\right)}{2(1+\beta)^2-2\beta^2\lambda^2}+\frac{g_u^2+g_v^2}2}\right]^n\right]\\
   &=\underset{\lambda}\E\left[ \underset{g_u,g_v}{\mathbb E}\left[\mathcal N\left((g_u,g_v);\Sigma\right)\mathcal N\left((g_u,g_v);\id_{2\times2}\right)^{-1}\right]^n\right].
   \end{aligned}
\end{equation}
where 
\begin{equation}
\Sigma^{-1}=\frac{1+\beta}{(1+\beta)^2-\beta^2\lambda^2}\begin{pmatrix}
    1+\beta & -\beta \lambda\\ -\beta \lambda& 1+\beta
\end{pmatrix}, \qquad \lambda= \frac{u\cdot v}d.
\end{equation}
Now we turn to compute the expectation over $\lambda$. Note that  since $u,v$ are independent and their components are Rademacher distributed, product of Rademacher is still Rademacher, hence $u\cdot v$ is the sum of $d$ independent Rademacher random variables. Using moreover that the Rademacher distribution is just a linear transformation of the Bernoulli, we can link the distribution of a the overlap $\lambda$ to a linear transformation of a binomial distribution: $\frac{d}{2}(\lambda+1) \sim \text{Binom}(d,1/2)$.

We therefore define the auxiliary function
\begin{equation}
  \label{eq:fdef1}
    f(\beta,\lambda):=\underset{g_u,g_v}{\mathbb E}\left[\frac{1+\beta}{\sqrt{(1+\beta)^2-\beta^2\lambda^2}}e ^{-\frac{(1+\beta)\left((1+\beta)(g_u^2+g_v^2)-2\beta(g_ug_v)\lambda\right)}{2(1+\beta)^2-2\beta^2\lambda^2}+\frac{g_u^2+g_v^2}2}\right]
\end{equation}
so that the LR norm can be rewritten as
\begin{equation}
     ||L_{n,d}||^{2}=\sum_{j=0}^d \binom{d}{j}\frac{1}{2^d}f\left(\beta,\frac {2j}d-1\right)^n
\end{equation}
which is what we needed to prove.

\subsubsection{Consequences of theorem \ref{thm:LRspiked_cumulant}}%
\label{subsubsec:Radem_non_gaus}
\begin{figure}[t]
\centering
  \includegraphics[width=0.8\textwidth]{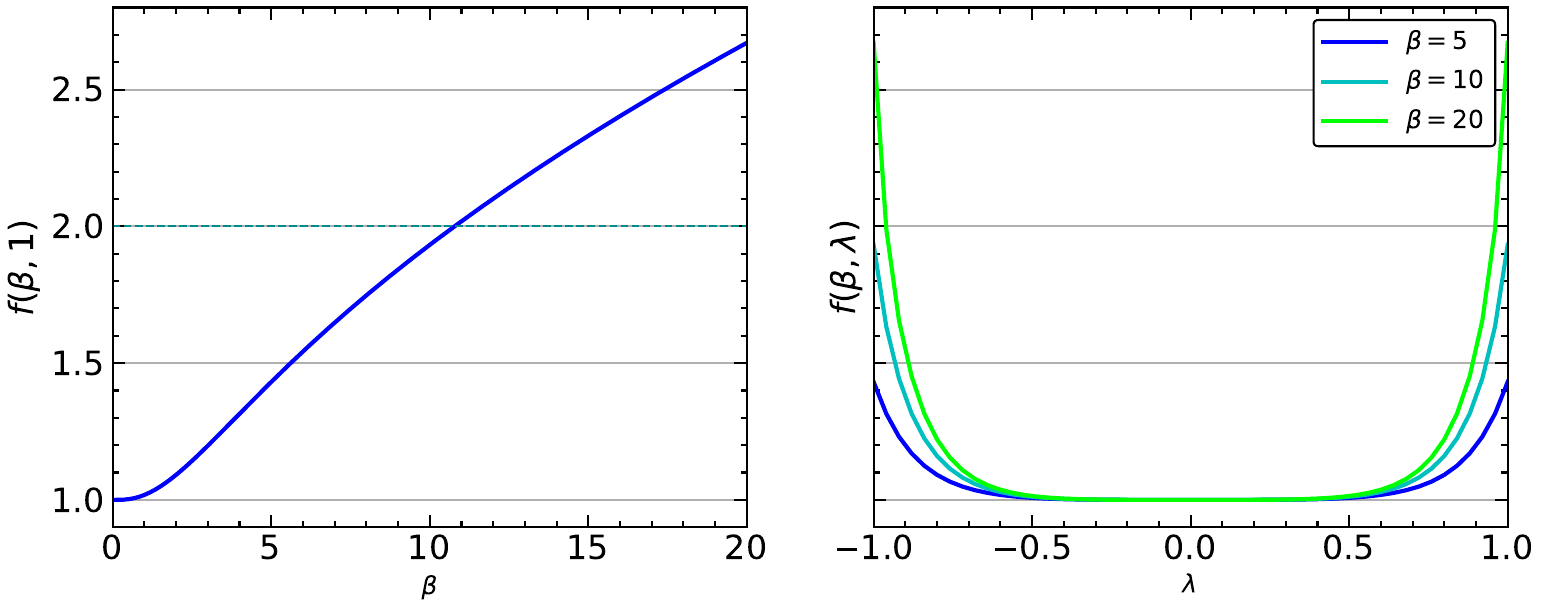}
\caption{Graphs of $f$, defined in \eqref{eq:fdef}, when $g\sim \text{Rademacher}(1/2)$.}
  \label{fig:f_radem}
\end{figure}

The precise value of $f$ depends on the choice of distribution for the non-Gaussianity $g$, however, it is possible to prove the following
\begin{lemma}
\label{lemma:theta<1}
    If $p_g$ satisfies assumptions \ref{assumptions:g} and $\theta<1$, then there exists $C$ such that:
    \[
    ||L_{n,d}||\le C \qquad \forall n,d
    \]
\end{lemma}
\begin{proof}
    We first prove that, thanks to the sub-Gaussianity of $g$, we have that
    \begin{equation}
    \label{eq:quadratic_estim_on_f}
        f(\beta,\lambda)\le 1+ C_\beta\lambda ^2, \qquad \forall \ \lambda\in [-1,1] 
    \end{equation}
    To do it, first note that $f\beta,0)=1$, and $f(\beta,\cdot)$ is bounded on $[-1,1]$ thanks to the sub-gaussianity of $g$:
    \begin{align*}
        |f(\beta,\lambda)|&\le \E_{\substack{g_u,g_v}}\left[e^{\frac12\left(g_u^2+g_v^2\right)}\right]\sup_{g_u,g_v}\left(\frac{1+\beta}{\sqrt{(1+\beta)^2-\beta^2\lambda^2}}e ^{-\frac{(1+\beta)\left((1+\beta)(g_u^2+g_v^2)-2\beta(g_ug_v)\lambda\right)}{2(1+\beta)^2-2\beta^2\lambda^2}}\right)\\
        &\le C \frac{1+\beta}{\sqrt{(1+\beta)^2-\beta^2\lambda^2}}\\
        &\le C_1 \left(1+C_2\lambda^2\right)
    \end{align*}
    So we just need to prove that, up to changes on $C_2$, we can take $C_1=1$. To do that it si sufficient to study  the behaviour around $\lambda=0$:
\[\frac{\partial}{\partial \lambda} f(\beta,0)=\E\left[\frac{\beta}{1+\beta}g_ug_v\right]=0
\]
So there is a neighborood of 0 such that $f(\beta,0)=1+C \lambda^2+o(\lambda^2)$. Hence we can take a suitable constant $C_\beta$ so that \cref{eq:quadratic_estim_on_f} holds.

    Now we can apply \cref{eq:quadratic_estim_on_f} to \eqref{eq:LRnormf}, to get
    \begin{align*}
    ||L_{n,d}||^2&\le \sum_{j=0}^d \binom{d}{j}\frac{1}{2^d}\left(1+C_\beta \left(\frac{2j}{d}-1\right)^2\right)^n\\
     &=\sum_{j=0}^d \binom{d}{j}\frac{1}{2^d}\sum_{k=0}^n \binom{n}{k} C_\beta^k\left(\frac{2j}{d}-1\right)^{2k}\\
&=\sum_{k=0}^n \binom{n}{k} \left(\frac{C_{\beta}}{d^2}\right)^k\EE[Y_d^{2k}]
    \end{align*}
    Where $Y_d$ is a random variable distributed as the sum of $d$ independent Rademacher with parameter $1/2$.
    So we can apply the central limit theorem on $\frac{Y_d}{\sqrt{d}}$ to get that $\frac{\EE[Y_d^{2k}]}{d}\to (2k-1)!!$ hence we get:
    \begin{align*}
        ||L_{n,d}||^2&\le C \sum_{k=0}^n \binom{n}{k} \left(\frac{C_{\beta}}{d}\right)^k (2k-1)!!\\
        &\le C_1  \sum_{k=0}^n\left(\frac{C_2n}{d}\right)^k
    \end{align*}
    where $C_1$ and $C_2$ are constants independent of $n,d$. So now we can substitute that $n\asymp d^\theta$, and since $\theta<1$ the series converges for all $d$, hence the LR norm is bounded.
\end{proof}
We will analyze in more detail the case in which $g\sim \text{Rademacher}(1/2)$. It is a case that is particularly interesting for the point of view of the applications because it amounts to comparing a standard Gaussian with a Gaussian mixture with same mean and covariance. Moreover, with this choice of non-Gaussianity, the technical work simplifies because the troublesome integral over $g_u,g_v$ in \ref{eq:fdef} becomes just a simple sum over 4 possibilities. In this case $f$ can be computed exactly and it is displayed in $\ref{fig:f_radem}$. The maximum of $f(\beta,\cdot)$ is attained at $\lambda=\pm 1$ and $f(\cdot,1)$ is monotonically increasing. 
Assume that $n\asymp d^\theta$, then a sufficient condition for LR  to diverge is $\frac{f(\beta,1)^{ d^\theta}}{2^d}\to \infty$, which holds as soon as $\theta>1$. 
Moreover, even at linear sample complexity,  it is possible to find regimes that ensure divergence of the LR norm, similar to BBP phase transition in spiked Wishart model.
Assume that samples and dimensions scale at the same rate as in spiked Wishart model: $n\asymp \frac{d}\gamma$, then a sufficient condition for divergence is that
\begin{equation*}
    \frac{f(\beta,1)^{d/\gamma}}{2^d}\to \infty
\end{equation*}
\begin{figure}[t]
\centering
  \includegraphics[width=0.8\textwidth]{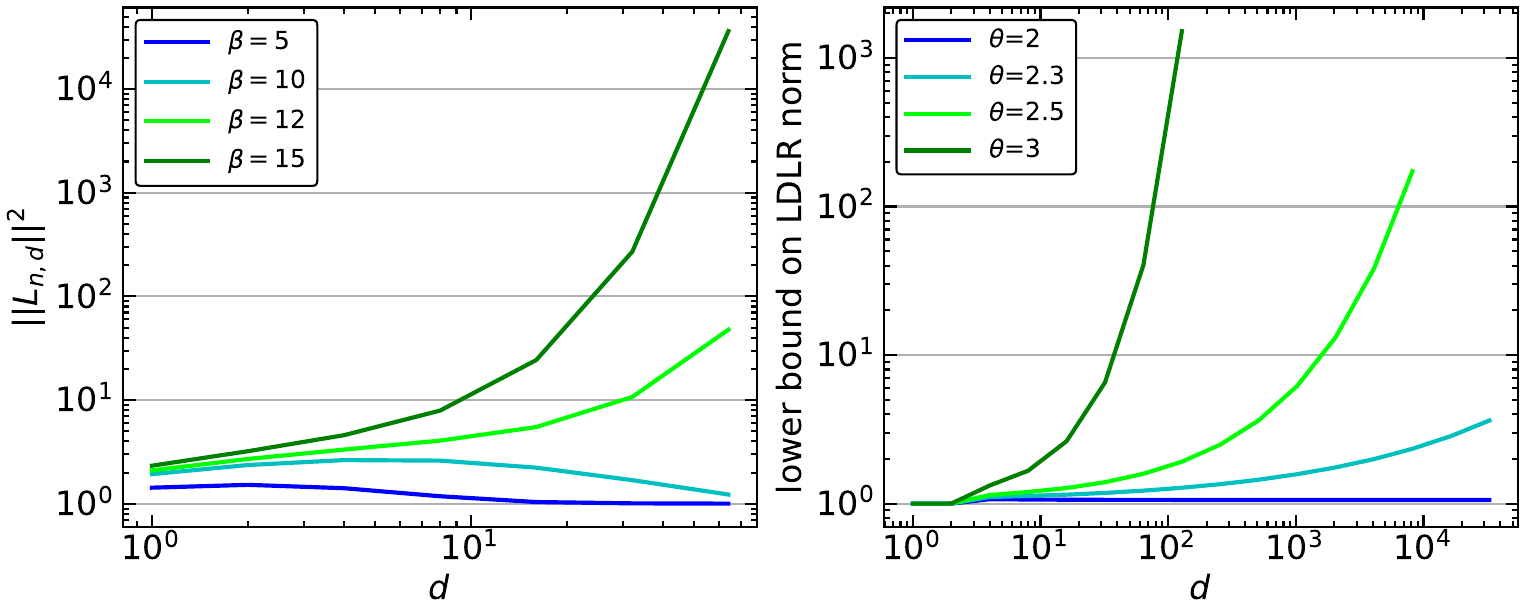}%
\caption{On the left, LR norm when $g\sim \text{Rademacher}(1/2)$ in the regime $n = \gamma d$ with $\gamma=1$. When $\beta < \beta_\gamma \approx 10.7$ the likelihood ratio remains bounded, whereas it goes to $+\infty$ for $\beta>\beta_{\gamma}$. On the right, the lower bound on $||L_{n,d}^{\le D(n)}||$ given by \ref{eq:LDLRestim} goes to $+\infty$ for $\theta>2$. Parameters for the plot: $g\sim $Radem$(1/2)$, $\beta=10$, $D(n)=\log^{3/2}(n)$}
  \label{fig:LR_spiked_cumulant}
\end{figure}
which holds if and only if $f(\beta,1)^{1/\gamma }> 2$. 
Hence given $\beta$, you can always find 
\begin{equation} \label{eq:gammabeta}\gamma_{\beta} := \frac{\log\left(f(\beta,1)\right)}{\log 2}
\end{equation}
And for $\gamma> \gamma_\beta$ there is guarantee that $||L_{n,d}||\to \infty$. Vice-versa, we could also fix $\gamma$ and define $\beta_\gamma$ as only value of $\beta$ that makes \ref{eq:gammabeta} true.

It is spontaneous to ask whether this condition for divergence of LR norm is also necessary, making the SNR threshold $\beta_\gamma$ the analogous in this model of the threshold $\beta_c=\sqrt{\gamma}$ in spiked Wishart model (\ref{thm:ldlr-wishart}).
Although a rigorous proof of this is still missing, numerical evidence (\ref{fig:LR_spiked_cumulant}) suggests that for $\beta\le \beta_\gamma$ the LR norm stays indeed bounded.

\subsection{LDLR on spiked cumulant model}%
\label{sec:cumulant-ldlr-proof}
To discuss the proof of \cref{thm:ldlr-cumulant}, it is best to state it in two separate parts 
\begin{theorem}[LDLR for spiked cumulant model]%
  \label{thm:ldlr-cumulant_old}
  Suppose that $(u_i)_{i=1,\dots,d}$ are drawn i.i.d.~from the symmetric
  Rademacher distribution. If the non-Gaussian distribution $p_g$ satisfies assumption
  \ref{assumptions:g}, then the following lower and upper bounds hold:
  \begin{itemize}
  \item Let $D\in \N$ such that $D/4\le n$, then:
    \begin{equation}
      \label{eq:LDLRestim}
      \norm{L_{n,d}^{\le D}}^2\ge \sum_{m=0}^{\left\lfloor D/4\right\rfloor} \binom{n}{m} \binom{d+1}{2}^m\left(\frac{\beta^2\kappa^{g}_4}{\sqrt{4!}d^2(1+\beta)^2}\right)^{2m}.
\end{equation}
where $\kappa_4^g$ is the fourth-order cumulant of $g$.
  \item Conversely, for any $D,n,d$:
    \begin{equation}
      \label{eq:LDLRupperbound}
      \norm{L_{n,d}^{\le D}}^2 \le
      1+\sum_{m=1}^D\frac{C_m}{d^{m}}\sup_{k\le
        m}\left(\E\left[h_{k}(g)\right]^{2m/ k}\right)\binom{n}{\lfloor
        m/4\rfloor} \binom{d}{\lfloor m/2\rfloor}
    \end{equation}
    where
      $C_m:= \left(\frac{\beta}{1+\beta}\right)^{m} \binom{\lfloor m/4\rfloor\lfloor
        m/2\rfloor+m-1}{m}$.
  \end{itemize}
\end{theorem}

\begin{corollary}[Asymptotics of LDLR bounds]%
  \label{cor:asymp_bound}
  Assume the hypotheses of theorem~\ref{thm:ldlr-cumulant}. Let $0<\varepsilon<1$ and
  assume $D(n)\asymp\log^{1+\varepsilon}(n)$. Take $n,d\to \infty$, with the
  scaling $n\asymp d^\theta$ for $\theta>0$. Estimate \eqref{eq:LDLRestim}
  implies that for $n,d$ large enough, the following lower bound holds:
  \begin{equation*}
    \norm{L_{n,d}^{\le D(n)}}^2 \ge\left( \frac 1 {\left\lfloor D(n)/4\right \rfloor}\left(\frac{\beta^2\kappa^{g}_4}{(1+\beta)^2}\right)^{2}\frac{n}{d^2}\right)^{\left\lfloor D(n)/4\right \rfloor}
  \end{equation*}
  Conversely, \eqref{eq:LDLRupperbound} leads to:
  \begin{equation*}
    \norm{L_{n,d}^{\le D(n)}}^2 \le
    1+\sum_{m=1}^{D(n)}\left(\frac{\Lambda^2\beta}{1+\beta}\right)^m
    m^{4m}\left(\frac{n}{d^2}\right)^{m/4}
  \end{equation*}
  Taken together, \eqref{eq:LDLR_asymp_lower_bound} and \eqref{eq:asymp_upper_bound} imply
  the presence of a critical regime for $\theta_c=2$, and describe the behaviour
  of $\norm{L_{n,d}^{\le D}}$ for all $\theta \ne \theta_c$
  \begin{equation*}\lim_{n,d\to
      \infty}\norm{L_{n,d}^{\le D(n)}}=\begin{cases}
      1 & 0<\theta< 2 \\
      +\infty & \theta> 2
    \end{cases}
  \end{equation*}
\end{corollary}

In the following we will present the proofs of \ref{thm:ldlr-cumulant_old} and
\ref{cor:asymp_bound}. We first give a sketch of the proof in
\ref{sec:cumulant-ldlr-sketch} before giving the detailed proof in
\ref{sec:cumulant-ldlr-details}.

\subsubsection{Proof sketch}%
\label{sec:cumulant-ldlr-sketch}

The starting point of the argument is the observation that since the null hypothesis is white Gaussian noise, the space $L^2(\R^{nd},\langle\cdot,\cdot\rangle)$ has an explicit orthogonal basis in the set of multivariate Hermite polynomials $(H_\ammult)$. The multi-index $\ammult\in \N^{nd}$ denotes the degree of the Hermite polynomial which is computed for each entry of the data matrix; see \ref{subsec:hermite} for a detailed explanation. We can then  expand the LDLR norm in this basis to write:
   \begin{equation} \label{eq:normLDLRgen}
\|L_{n,d}^{\le D}\|^2=\underset{|\ammult|\le D}{\sum}\frac {\langle L,H_{\ammult}\rangle^2}{\ammult!}=\underset{|\ammult|\le D}{\sum}\frac{1}{\ammult!} \; {\underset{\ux \sim \uP}{\E}[H_{\ammult}(\ux)]^2}
\end{equation}
From now on the two bounds are treated separately.

\paragraph{Lower bound}

The idea is to bound the LDLR from below by computing the sum \ref{eq:normLDLRgen} only over a restricted set of addends $\mathcal{I}_m$ which we can compute explicitly. In particular, we consider the set $\mathcal I_m$ of all the polynomials with degree $4m$ in the data matrix, which are at most of order 4 in each individual sample $x^\mu$. Then we use that the expectation of such Hermite polynomials conditioned on $u$ can be split into $m$ integrals of degree-4 Hermite polynomials in $d$ variables. In this way we can use our knowledge of the fourth-order cumulant of $x$ to compute those expectations (see \ref{eq:LDLRorder4}), so we find that
\begin{equation}
\|L_{n,d}^{\le D}\|^2
\ge \sum_m^{\left\lfloor D/4\right\rfloor} \underset{|\ammult|\in \mathcal{I}_m }{\sum}\frac{1}{\ammult!} \; {\underset{\ux \sim \uP}{\E}[H_{\ammult}(\ux)]^2} 
\ge \sum_m^{\left\lfloor D/4\right\rfloor} \underset{\ammult\in \mathcal{I}_m }{\sum} \left(\frac{\beta^2\kappa^{g}_4}{\sqrt{4!}d^2(1+\beta)^2}\right)^{2m}.
\end{equation}
Thanks to this manipulation, the bound does not depend on $\ammult$ explicitly, so we can complete the bound by estimating the cardinality of $\mathcal{I}_m$. Thanks to the fact that we have $n$ i.i.d.\ copies of variable $x$ at disposal, the class of polynomials $\mathcal I_m$ has size that grows at least as $\binom{n}{m} \binom{d+1}{2}^m$, allowing to reach the lower bound in the statement.

\paragraph{Upper bound}

We can show that the expectation of Hermite polynomials for a single sample $x^\mu$ over the distribution $\mathbb{P}(\cdot | u)$ can be written as $\E_{x \sim \P(\cdot|u)}\left[H_{\amult}(x)\right]=  T_{|\amult|,g} / {d^{|\amult|/2}}u^\amult$, where $T_{|\amult|,g}=\left(\frac{\beta}{1+\beta}\right)^{|\amult|/2}\mathbb E\left[h_{|\amult|}\left(g\right)\right]$, cf.~lemma \ref{lemma:hermitepartupperbound}. Using the fact that inputs are sampled i.i.d.~from $\mathbb{P}(\cdot | u)$ and substituting this result into \ref{eq:normLDLRgen}, we find that
\begin{equation} \label{eq:upperbound_decoupled_expectations}
\|L_{n,d}^{\le D}\|^2=\sum_{|\ammult|\le D}\frac {\left(\prod_{\mu=1}^{n}  T_{|\amult^\mu|,g}\right)^2}{\ammult!d^{|\ammult|}} 
\underset{u \sim \mathcal P(u)}{\E}\left[ u^\ammult \right]^2.
\end{equation}
To obtain an upper bound, we will first show that many addends in this sum are equal to zero, then estimate the remainder. 

On the one hand, since we chose a Rademacher prior over the elements of the spike $u_i$, the expectation over the prior yields either 1 or 0 depending on whether the exponent of each $u_i$ is even or odd. On the other hand, the whitening of the data means that $T_{|\amult|,g}=0$ for $0 < |\amult| < 4$ (as proved in lemma \ref{lemma:hermitepartupperbound}).

For each $m\in \N$, we can denote $\mathcal{A}_m$ the set of multi-indices $|\ammult|=m$ that give non-zero contributions in \ref{eq:upperbound_decoupled_expectations}, so that we can write:
\begin{equation}
     \|L_{n,d}^{\le D}\|^2 = \sum_{m=0}^{D} \sum_{\ammult \in \mathcal{A}_m} \frac {\left(\prod_{\mu=1}^{n}  T_{|\amult^\mu|,g}\right)^2}{\ammult!d^{m}} 
\end{equation}
Now we proved that we can bound the inner terms so that they depend on $\ammult$ only through the norm $|\ammult|=m$ (cf.\ \ref{eq:Tmformula} and \ref{eq:LDLRintermediateestim}):
\begin{equation}
      \begin{aligned}\|L_{n,d}^{\le D}\|^2
        &\le \sum_{m=0}^D\sum_{\ammult\in \mathcal A_m}\frac1{d^{m}}  \left(\frac{\beta}{1+\beta}\right)^{m}\sup_{k\le m}\left(\E\left[h_{k}(g)\right]^{2m/k}\right)\\
        &= \sum_{m=0}^D \frac1{d^{m}}  \left(\frac{\beta}{1+\beta}\right)^{m}\sup_{k\le m}\left(\E\left[h_{k}(g)\right]^{2m/k}\right) \# \mathcal A_m
     \end{aligned}
     \end{equation}
Finally, the cardinality of $\mathcal A_m$ is 1 in case $m=0$ (which leads to the addend 1 in \ref{eq:LDLRupperbound}) and if $m>0$ it can be bounded by (for details see \ref{eq:tilde_A_m} in the appendix) 
\begin{equation} \binom{\lfloor m/4\rfloor\lfloor m/2\rfloor+m-1}{m}\binom{n}{\lfloor m/4\rfloor} \binom{d}{\lfloor m/2\rfloor}, 
\end{equation}
leading to \eqref{eq:LDLRupperbound}.

\subsubsection{Detailed proof}%
\label{sec:cumulant-ldlr-details}

First we  prove a lemma that provides formulas and estimates for projections of the sample-wise likelihood ratio $l(\cdot|u)$ on the Hermite polynomials.
\begin{lemma} \label{lemma:hermitepartupperbound}
    Let $x=S\left(\sqrt{\beta/d}g u+z\right)$ be a spiked cumulant random variable. Then for any $\amult \in \N^{d}$, with $|\amult|=m$, we have that:
    \begin{equation}\label{eq:hermite_exp_multiple_of_u}
\left \langle l(\cdot|u),H_\amult\right \rangle  =  \underset{x \sim \P(\cdot|u)}{\E}\left[H_{\amult}(x)\right]=  \frac{T_{m,g}}{d^{m/2}}u^\amult
    \end{equation}
    where $T_{m,g}$ is a coefficient defined as:
    \begin{equation}\label{eq:Tmformula}
    T_{m, g}=\left(\frac{\beta}{1+\beta}\right)^{m/2}\mathbb E\left[h_{m}\left(g\right)\right]
     \end{equation}
   
\end{lemma}
\begin{proof}
Recall that by \eqref{eq:integraloverlambda}
\begin{equation} 
    l(y|u)= \underset{g} {\E}\left[\sqrt{1+\beta} \exp\left(-\frac{1}{2}\left(\sqrt{1+\beta}g-\sqrt{\frac \beta {d}}y\cdot u\right)^2+\frac{g^{2}}{2}\right)\right].
\end{equation}
Note that this expression, thanks to \ref{eq:assumpTails} that bounds the integral, is differentiable infinitely many times in the $y$ variable.
It can be quickly proven by induction on $|\amult|=m$ (using the recursive definition of Hermite polynomials \ref{eq:HermPol_recur}) that:
\begin{multline}\label{eq:partialLfirstexpr}
    \partial_\amult l(y|u)=\left(\frac{\beta}{d}\right)^{m/2}u^\amult \cdot \\ \cdot \underset{g}{\E}\left[\sqrt{1+\beta} \ h_{m}\left(\sqrt{1+\beta}g-\sqrt{\frac{\beta}{d}}y\cdot u\right)\exp\left(-\frac12 \left(\sqrt{1+\beta}g-\sqrt{\frac{\beta}{d}}y\cdot u\right)^2+\frac{g^2}2\right)\right]
\end{multline}
Hence we can again use \ref{eq:assumpTails} together with the fact that 
\[\sup_g \left |h_{m}\left(\sqrt{1+\beta}g-\sqrt{\frac{\beta}{d}}y\cdot u\right)\exp\left(-\frac12 \left(\sqrt{1+\beta}g-\sqrt{\frac{\beta}{d}}y\cdot u\right)^2\right)\right|<\infty
\]
to deduce that the hypothesis of lemma \ref{lemma:IntByParts} are met for any $\amult\in \N^d$, leading to:
\begin{equation}
    \left \langle l(\cdot|u),H_\amult\right \rangle  =  \underset{x \sim \P(\cdot|u)}{\E}\left[H_{\amult}(x)\right]=\underset{y\sim \Q}{\E}\left[\partial_{\amult}l(y|u)\right]
\end{equation}
This, and \ref{eq:partialLfirstexpr} already prove \ref{eq:hermite_exp_multiple_of_u}. Now to compute the exact value of $T_{m,g}$ we need to take the expectation with respect to $y\sim \Q=\mathcal N(0,\id_{d\times d})$. Note that by the choice of Rademacher prior on $u$, we know that $||u||=\sqrt d$. So, conditioning on $u$, $y\cdot \frac{u}{\sqrt d}\sim \mathcal N(0,1)$. Hence switching the expectations in \ref{eq:partialLfirstexpr}, we get:
\begin{equation}
\begin{aligned}
    T_{m,g}&=\beta^{m/2}\underset{g}{\E}\left[\int_{-\infty}^\infty \frac{\deriv z}{\sqrt{2\pi}}\sqrt{1+\beta} \ h_{m}\left(\sqrt{1+\beta}g-\sqrt{\beta}z \right)e^{-\frac12 \left(\sqrt{1+\beta}g-\sqrt{\beta}z\right)^2+\frac{g^2-z^2}2}\right]\\
    &=\beta^{m/2}\underset{g}{\E}\left[\int_{-\infty}^\infty \frac{\deriv z}{\sqrt{2\pi}}\sqrt{1+\beta} \ h_{m}\left(\sqrt{1+\beta}g-\sqrt{\beta}z \right)\exp\left(-\frac12 \left(\sqrt{\beta}g-\sqrt{1+\beta}z\right)^2\right)\right]\\
   &\overset{\tilde z=\sqrt{1+\beta}z}{=}\beta^{m/2}\underset{g}{\E}\left[\int_{-\infty}^\infty \frac{\deriv \tilde z}{\sqrt{2\pi}} \ h_{m}\left(\sqrt{1+\beta}g-\sqrt{\frac{\beta}{1+\beta}}\tilde z \right)\exp\left(-\frac12 \left(\sqrt{\beta}g-\tilde z\right)^2\right)\right]
\end{aligned}
\end{equation}
Now we use \ref{eq:HermPolBinomId_sum} on 
\begin{equation}
    \begin{aligned}
        x&=\frac{\beta}{\sqrt{1+\beta}}g-\sqrt{\frac \beta{1+\beta}}\tilde z\\
        y&=\sqrt{1+\beta}g-\frac{\beta}{\sqrt{1+\beta}}g=\frac{g}{\sqrt{1+ \beta}}
    \end{aligned}
\end{equation}
applying also the translation change of variable $\hat z=\tilde z-\sqrt\beta \, g$ we get:
\begin{equation}
     \underset{y\sim \Q}{\E}\left[\partial_{\amult}l(y|u)\right]=
       \left(\frac{\beta}{d}\right)^{m/2}u^\amult\underset{g}{\E}\left[\sum_{k=0}^m\binom{m}{k} h_k\left(\frac{g}{\sqrt{1+ \beta}}\right)\left(-\sqrt{\frac{\beta}{1+\beta}} \right)^{m-k}\underset{\hat z\sim \mathcal N(0,1)}{\E}[\hat z^{m-k}]\right]
\end{equation}
Now recall the formula for the moments of the standard Gaussian (see for instance section 3.9 in \citep{mccullagh2018tensor}):
\begin{equation}
    \underset{\hat z\sim \mathcal N(0,1)}{\E}[\hat z^{m-k}]=\begin{cases}
        (m-k-1)!! & \text{if $m-k$ is even}\\
        0 & \text{if $m-k$ is odd }
    \end{cases}
\end{equation}
Plugging this formula inside and changing the summation index $2j:=m-k$ we get 
 \begin{equation} \label{eq:Tlongexpr}
    T_{m, g}=\beta^{m/2}\sum_{j=0}^{\lfloor m/2\rfloor} \binom{m}{2j} \left(\frac{\beta}{1+\beta}\right)^j (2j-1)!!\mathbb E\left[h_{m-2j}\left(\frac{g}{\sqrt{1+\beta}}\right)\right].
\end{equation}
Note that \ref{eq:HermPolBinomId_scalmult} with $x=\frac{g}{\sqrt{1+\beta}}$ and $\gamma=\sqrt{1+\beta}$ gives the following rewriting of $h_m(g)$:
\begin{equation}
     h_m(g)=\sum_{j=0}^{\lfloor m/2\rfloor} (\sqrt{1-\beta})^{m-2j}(\beta)^j\binom{m}{2j}  (2j-1)!! h_{m-2j}\left(\frac{g}{\sqrt{1+\beta}}\right)
\end{equation}
which is almost the same expression as in \ref{eq:Tlongexpr}. This allows simplify everything, leading to \ref{eq:Tmformula}. 
\end{proof}

 Note that lemma \ref{lemma:hermitepartupperbound} together with \ref{assumptions:g} on $g$ imply:
     \begin{equation} \label{eq:Tzero}
        m=2\text{ or $m$ odd}\Longrightarrow T_{m,g}=0,
    \end{equation}
     So, apart from $m=0$ which gives the trivial contribution +1, the first non zero contributions to the LDLR norm is at degree $m=4$:
    \begin{equation} \label{eq:LDLRorder4} \underset{x\sim \P(\cdot|u)}\E[H_\amult(x)]=\frac{\beta^2}{(1+\beta)^2}\kappa^g_4 \frac{u^\amult}{d^2}
    \end{equation}

From now on we will consider separately the lower and the upper bounds.

\paragraph{Lower bound}

The idea of the proof is to start from
\begin{equation} ||L_{n,d}^{\le D}||^2=\underset{|\ammult|\le D}{\sum_{\ammult\in \mathbb N^{nd}}}\frac {\langle L,H_{\ammult}\rangle^2}{\ammult!}\\
\end{equation}

and to estimate it from below by picking only few terms in the sum. So we
restrict to particular sets of multi-multi-indices for which we can exploit \ref{eq:LDLRorder4}. Let $m\in \N$ such that $4m\le D$, define
\begin{equation}\mathcal I_{m}=\left\{\ammult\in \left(\N^{d}\right)^n\Big||\ammult|=4m,\ |\amult^\mu|\in\{0,4\} \ \forall \ 1\le \mu\le n\right\}
\end{equation}
$\mathcal I_{m}$ is non empty since $m<n$ and for each $\ammult\in \mathcal I_{m}$ we can enumerate all the indices $\mu_1,\dots,\mu_m$ such that $\amult^{\mu_i}\ne 0$.

Now we go on to compute the term $\langle L,H_{\ammult}\rangle$ for $\ammult\in \mathcal I_m$:
\begin{equation} \begin{aligned}
    \langle L,H_{\ammult}\rangle&=\int_{\mathcal U} \underset{\ux\sim\P(\cdot|u)^{\otimes n}}{\E}[H_{\ammult}(\ux)] \deriv \mathcal P(u)\\
    &=\int_{\mathcal U} \underset{\ux\sim\P(\cdot|u)^{\otimes n}}{\E}\left[\prod_{i=1}^{m}H_{\amult^{\mu_i}}(x^{\mu_i})\right] \deriv \mathcal P(u)\\
\end{aligned}\end{equation} 
Since the samples are independent conditionally on $u$, we can split the inner expectation along the $m$ contributing directions:
\begin{equation} \begin{aligned}
    \langle L,H_{\ammult}\rangle&=\int_{\mathcal U} \prod_{i=1}^{m}\underset{x\sim\P(\cdot|u)}{\E}[H_{\amult^{\mu_i}}(x^{\mu_i})] \deriv \mathcal P(u)
\end{aligned}\end{equation} 
Now we need to compute the inner $d$-dimensional expectation. For that we use that $|\amult^{\mu_i}|=4$, so, recalling the notation $\eta=\sqrt{\frac{\beta}{1+\beta}}$, we can apply \ref{eq:LDLRorder4} to get for each $i$
\begin{equation}\underset{x\sim\P(\cdot|u)}{\E}[H_{\amult^{\mu_i}}(x^{\mu_i})]= \frac{\eta^4}{d^2} \kappa_4^g u^{\alpha^{\mu_i}}.\end{equation}
So the resulting integral can be written in the following way: 
\begin{equation} \begin{aligned}
    \langle L,H_{\ammult}\rangle&=\left(\frac{\eta^4\kappa_4^g}{d^2}\right)^m\int_{\mathcal U} \prod_{j=1}^{d} u^{\gamma_j}_{j}\deriv \mathcal P(u)
\end{aligned}\end{equation} 
where for each $j\in \left\{ 1,\dots,d\right\}$, $\gamma_j:=\sum_{\mu} \amult^{\mu}_j$. So we also have that $\sum_j \gamma_j=4m$.

Now we take the expectation with respect to $\mathcal P(u)$, and use the fact that the components are i.i.d.\ Rademacher so the result depends on the parity of the $\gamma_j$ in the following way:
\begin{equation}
  \langle L,H_{\ammult}\rangle= \begin{cases}
      \left(\frac{\eta^4\kappa_4^g}{d^2}\right)^m & \text{if all $(\gamma_j)_{j=1,\dots,n}$ are even} \\
      0 & \text{if there is at least one $\gamma_j$ which is an odd number} \\
  \end{cases}
\end{equation}

Hence, if we restrict to the set:
\begin{equation}
\tilde {\mathcal I}_m=\left\{\ammult \in \mathcal I_m\big| \forall j \in \{1,\dots,d\} \, \sum_{\mu=1}^n\ammult^\mu_j \text{ is even} \right\}
\end{equation}
We have that all the indices belonging to $\tilde{\mathcal I}_m$ give the same contribution:
\begin{equation}
\langle L,H_{\ammult}\rangle^2=\left(\frac{\eta^4}{d^2}\kappa_4^g\right)^{2m}
\end{equation}

Also, note that inside $\mathcal I_m$, $\amult!\le (4!)^m$, so get the following estimates

\begin{equation} \begin{aligned} ||L_{n,d}^{\le D}||^2&=\sum_{|\ammult|\le D}\frac {\langle L,H_{\ammult}\rangle^2}{\ammult!}\\
&\ge \sum_{m=0}^{\left\lfloor D/4\right\rfloor}\sum_{\ammult\in \tilde{\mathcal I}_m}\frac {\langle L,H_{\ammult}\rangle^2}{\ammult!}\\
&\ge\sum_{m=0}^{\left\lfloor D/4\right\rfloor}\sum_{\ammult\in \tilde{\mathcal I}_m}  \left(\frac{\eta^4\kappa_4^g}{\sqrt{4!}d^2}\right)^{2m}\\
&\ge\sum_{m=0}^{\left\lfloor D/4\right\rfloor} \#\tilde{\mathcal I}_m  \left(\frac{\eta^4\kappa_4^g}{\sqrt{4!}d^2}\right)^{2m}
\end{aligned}\end{equation} 
Now we just need to estimate the cardinality of $\tilde {\mathcal I}_m$. A lower bound can be provided by considering 
 \begin{equation}
  \hat{\mathcal I}_m=  \left\{\ammult \in \mathcal I_m\big| \  \forall \ \mu\in [n] \ \forall  \ j\in [d]  \ \ammult^\mu_j \text{ is even} \right\}
  \end{equation}
   Clearly $\hat{\mathcal I}_m\subseteq \tilde{\mathcal I_m}$
  and $\#\hat{\mathcal I_m}=   \binom{n}{m}\binom{d+1}{2}^m$ because first we can pick the $n-m$ rows of $\ammult$ that will be $0\in \N^d$, which can be done in $\binom{n}{m}$. Then, for each  non-zero row, we need to pick two columns (with repetitions) in which to place a 2 and leave all the other entries as 0, that can be done in $\binom{d+1}{2}^m$ ways.



So plugging the lower bound in the previous estimate, we reach the inequality that we wanted to prove:
   \begin{equation}||L_{n,d}^{\le D}||^2 \ge \sum_{m=0}^{\left\lfloor D/4\right\rfloor}  \binom{n}{m}\binom{d+1}{2}^m \left(\frac{\eta^4\kappa^{g}_4}{\sqrt{4!}d^2}\right)^{2m}=\sum_{m=0}^{\left\lfloor D/4\right\rfloor}  \binom{n}{m}\binom{d+1}{2}^m \left(\frac{\beta^2\kappa^{g}_4}{\sqrt{4!}d^2(1+\beta)^2}\right)^{2m}
   \end{equation}

\paragraph{Upper bound}
    We start from \eqref{eq:normLDLRgen} using the following rewriting:
     \begin{equation}
     \begin{aligned}
    ||L_{n,d}^{\le D}||^2&=\sum_{|\ammult|\le D}\frac {\langle L,H_{\ammult}\rangle^2}{\ammult!}\\
    &=\sum_{|\ammult|\le D}\frac {\underset{\ux \sim \uP}{\E}\left[H_{\ammult}(\ux)\right]^2}{\ammult!}\\
     &=\sum_{|\ammult|\le D}\frac 1{\ammult!}\underset{u \sim \mathcal P(u)}{\E}\left[\prod_{\mu=1}^{n}\underset{x^{\mu} \sim \P(\cdot|u)}{\E}\left[H_{\amult^\mu}(x^\mu)\right]\right]^2
     \end{aligned}
    \end{equation}
     Now use \ref{lemma:hermitepartupperbound} and plug in the formulas for the inner expectations.
     \begin{equation} \label{eq:LDLRformulabeforereducing}
     \begin{aligned}
    ||L_{n,d}^{\le D}||^2&=\sum_{|\ammult|\le D}\frac 1{\ammult!}\underset{u \sim \mathcal P(u)}{\E}\left[\prod_{\mu=1}^{n}  \frac{T_{|\amult^\mu|,g}}{d^{|\amult^\mu|/2}}u^{\amult^\mu}\right]^2\\
    &=\sum_{|\ammult|\le D}\frac {\left(\prod_{\mu=1}^{n}  T_{|\amult^\mu|,g}\right)^2}{\ammult!d^{|\ammult|}}\underset{u \sim \mathcal P(u)}{\E}\left[u^\ammult\right]^2
     \end{aligned}
    \end{equation}  
    Now we use our prior assumption that $u_i\overset{\text{i.i.d.}}{\sim} \text{Rad}(1/2)$. Note that odd moments of $u_i$ are equal to 0 and even moments are equal to 1, so (denoting by $\chi_A(\cdot)$ the indicator function of set $A$):
    \begin{equation}
        \underset{u \sim \mathcal P(u)}{\E}\left[u^\ammult\right]=\prod_{i=1}^{d}\underset{u_i \sim\text{Rademacher}(1/2)}{\E}\left[u_i^{\sum_{\mu=1}^n\amult^{\mu}_i}\right]= \chi_{\left\{\sum_{\mu=1}^n \amult_i^{\mu} \text{ is even for all }i\right\}}(\ammult)
    \end{equation}
    Now the key point of the proof: this last formula, together with 
    \eqref{eq:Tzero}, implies that most of the addends in the sum in \eqref{eq:LDLRformulabeforereducing} are zero. Set $|\ammult|=m$, then the set of multi-indices that could give a non-zero contribution is
    \begin{equation} \label{eq:Amdef}
    \mathcal A_m=\left\{\ammult \in \N^{n\times d} \big | |\ammult|=m,\, \forall \  \mu\in [n], \ammult^\mu=0 \text{ or }  |\ammult^\mu|>4, \text{ and } \forall \ i\in [d] \sum_{\mu=1}^n \amult_i^{\mu} \text{ is even }\right\}
    \end{equation}

    Using this fact together with \eqref{eq:Tmformula} we get:
    \begin{equation}\label{eq:LDLRintermediateestim}
      \begin{aligned}||L_{n,d}^{\le D}||^2&=\sum_{m=0}^D\sum_{\ammult\in \mathcal A_m}\frac{\left(\prod_{\mu=1}^{n}  T_{|\amult^\mu|,g}\right)^2}{\ammult!d^{|\ammult|}}\\
      &\le \sum_{m=0}^D\sum_{\ammult\in \mathcal A_m}\frac1{d^{m}} \prod_{\mu=1}^n \left( \left(\frac{\beta}{1+\beta}\right)^{|\amult^\mu|}\E\left[h_{|\amult^\mu|}(g)\right]^2\right)\\
        &\le \sum_{m=0}^D\sum_{\ammult\in \mathcal A_m}\frac1{d^{m}}  \left(\frac{\beta}{1+\beta}\right)^{m}\sup_{k\le m}\left(\E\left[h_{k}(g)\right]^{2m/k}\right)\\
        &= \sum_{m=0}^D \frac1{d^{m}}  \left(\frac{\beta}{1+\beta}\right)^{m}\sup_{k\le m}\left(\E\left[h_{k}(g)\right]^{2m/k}\right) \# \mathcal A_m
     \end{aligned}
     \end{equation}
     In the third step we used the following inequality
     \[
     \prod_\mu \E\left[h_{|\amult^\mu|}(g)\right]^2\le  \prod_\mu \sup_{k\le m}\left(\E\left[h_{k}(g)\right]^{1/k}\right)^{2|\amult^\mu|}= \sup_{k\le m}\E\left[h_{k}(g)\right]^{2m/k}
     \]
     That holds since $\ammult\in \mathcal A_m$ hence $\sum_\mu |\amult^\mu|=m$.
     
         It is hard to compute exactly the cardinality of $\mathcal A_m$, but we can estimate it by considering the inclusion $\mathcal A_m\subseteq \tilde {\mathcal A}_m$ defined as
         \begin{equation} \label{eq:tilde_A_m}
             \tilde{\mathcal A}_m=\left\{\ammult \in \N^{n\times d}\big | |\ammult|=m, \, \text{there are at most $\left\lfloor \frac m4\right\rfloor$ non-zero rows and  $\left\lfloor \frac m2\right\rfloor$ non-zero columns}\right\}
         \end{equation}
         Assume now $m>0$ (the case $m=0$ can be treated separately, leading to the addend  +1 in \eqref{eq:LDLRupperbound}).
         To compute the cardinality of $  \tilde{\mathcal A}_m$ we just have to multiply the different contributions
         \begin{itemize}
             \item There are $ \binom{n}{\lfloor m/4\rfloor}$ possibilities for the non-zero rows
             \item There are $\binom{d}{\lfloor m/2\rfloor}$ possibilities for the non-zero columns
             \item once restricted to a $\lfloor m/4\rfloor\times \lfloor m/2\rfloor$ we have to place the units to get to norm $m$. It is the counting problem of placing $m$ units inside the $\lfloor m/4\rfloor\lfloor m/2 \rfloor$ matrix entries. The possibilities are 
             \begin{equation}\binom{\lfloor m/4\rfloor\lfloor m/2\rfloor+m-1}{m}
         \end{equation}
         \end{itemize}
         So we get the following estimate for the cardinality of $\mathcal A_m$
    \begin{equation}
        \# \mathcal A_m\le \binom{n}{\lfloor m/4\rfloor} \binom{d}{\lfloor m/2\rfloor} \binom{\lfloor m/4\rfloor\lfloor m/2\rfloor+m-1}{m}
    \end{equation}
    Plugging into \eqref{eq:LDLRintermediateestim} we reach the final formula \eqref{eq:LDLRupperbound}, which completes the proof.

\paragraph{Proof of corollary~\ref{cor:asymp_bound}}
Now we turn to the proof of \ref{cor:asymp_bound} on the asymptotic behavior of the bound

\begin{equation}
\binom{n}{m}=\frac{n^m}{m!}+O(n^{m-1})
\end{equation}
\begin{equation}
\binom{d+1}{2}^m\ge\frac{d^{2m}}{2^m}
\end{equation}
So we have that:
\begin{equation}
\begin{aligned}
&||L_{n,d}^{\le D(n)}||^2 \ge \sum_{m=0}^{\left\lfloor D/4\right\rfloor} \left(\frac{n^m}{m!}+O(n^{m-1})\right)\left(\frac{d^{2m}}{2^m}+O(1)\right)\left(\frac{\beta^2\kappa^{g}_4}{\sqrt{4!}d^2(1+\beta)^2}\right)^{2m}\\
&=\left(\sum_{m=0}^{\left\lfloor D/4\right\rfloor} \frac{1}{m!}\left(\frac{\beta^2\kappa^{g}_4}{\sqrt 2\sqrt{4!}(1+\beta)^2}\right)^{2m}\left(\frac{n}{d^2}\right)^{m}\right) +O \left(\sum_m\frac{n^{m-1}}{m!d^{2m}}\left(\frac{\beta^2\kappa^{g}_4}{\sqrt 2\sqrt{4!}(1+\beta)^2}\right)^{2m} \right)\\
\end{aligned}
\end{equation}
Then we lower-bound the sum by considering just the last term.
\begin{equation}
    ||L_{n,d}^{\le D(n)}||^2 \ge \frac{1}{\left\lfloor D/4\right \rfloor!}\left(\frac{\beta^2\kappa^{g}_4}{\sqrt 2\sqrt{4!}(1+\beta)^2}\right)^{2\left\lfloor D/4\right \rfloor}\left(\frac{n}{d^2}\right)^{\left\lfloor D/4\right \rfloor} +o \left(\frac{1}{\left\lfloor D/4\right \rfloor!}\left( \frac{Cn}{d^2} \right)^{\left\lfloor D/4\right \rfloor}\right)
\end{equation}
So by using $k!\le k^{k}$ on $\left\lfloor D/4\right \rfloor$ and picking $n,d$ large enough so that numerical constants and the $o(\dots)$  become negligible for the estimate, we get \eqref{eq:LDLR_asymp_lower_bound}. 

Plugging in the scaling $n\asymp d^{\theta}$ and $D(n)\asymp \log^{1+\varepsilon}(n)$ it is clear that what decides the behaviour of the sequence is the term $\frac{n}{d^{2}}$. 
\begin{itemize}
    \item If $0<\theta \le 2$, $\frac{n}{d^{2}}\to 0$ and so \eqref{eq:LDLR_asymp_lower_bound} does not provide information on  the divergence of the LDLR norm. 
    \item If $\theta >2$ it is $\frac{n}{d^{2}}\to \infty$ faster that the logarithmic term at denominator, so the right hand side in \eqref{eq:LDLR_asymp_lower_bound} diverges at $\infty$, proving the second regime in \eqref{eq:LDLR_final_asympt}
\end{itemize}

Now let us turn to proving \eqref{eq:asymp_upper_bound}. 
We are in the regime $m\le D << \min(n,d)$, hence we can use the following estimates of binomial coefficients and factorial 
 \begin{equation}
     \begin{aligned}
         \binom{n}{\lfloor m/4\rfloor}&\le  n^{ m/4}\\
         \binom{d}{\lfloor m/2\rfloor}&\le  d^{ m/2}\\
         \binom{\lfloor m/4\rfloor\lfloor m/2\rfloor+m-1}{m}&\le \left(m^2/8\right)^m\\
         m!&\le m^m
     \end{aligned}
 \end{equation}
 on \ref{eq:LDLRupperbound} to get
 \begin{equation}
     ||L_{n,d}^{\le D}||^2
      \le1+\sum_{m=1}^D\left(\frac{\beta}{1+\beta}\right)^m m^{2m}\sup_{ k\le m}\left(\E\left[h_k(g)\right]^{2m/ k}\right)\left(\frac{n}{d^2}\right)^{m/4}
 \end{equation}
 Now we use estimate the term depending on the Hermite coefficients of $g$, using the assumption \ref{eq:assumpHermCoef}:
 \begin{equation}
     \begin{aligned}
        \sup_{ k\le m}\left(\E\left[h_k(g)\right]^{2m/ k}\right) &\le \sup_{ k\le m}\left((\Lambda^k k!)^{2m/ k}\right)\\
        &\le \Lambda^{2m} \sup_{ k\le m}\left(k^{2m}\right)\\
        &\le \Lambda^{2m}m^{2m}
     \end{aligned}
 \end{equation}
    Plugging into the estimate for the LDLR we get \eqref{eq:asymp_upper_bound}
   \begin{equation}
        ||L_{n,d}^{\le D}||^2
      \le 1+\sum_{m=1}^D\left(\frac{\Lambda^2\beta}{1+\beta}\right)^m m^{4m}\left(\frac{n}{d^2}\right)^{m/4}
   \end{equation}
 By letting  $n\asymp d^{\theta} $, for $\theta>0$ and $D(n)\asymp \log^{1+\varepsilon}(n)$, we can see that the bound goes to 1 when $\theta < 2$, proving the first regime in \eqref{eq:LDLR_final_asympt}.

\subsection{Limitations of our theoretical analysis}

The main limitation of the theoretical portion of our work is that it relies
on the assumption that the null hypothesis is standard i.i.d.\ Gaussian
noise. Since this assumption is central to the large body of work analysing
hypothesis tests~\cite{kunisky2019notes}, a very interesting future direction
is to develop tools for analysing the case of a different null hypothesis.  On
the technical side, while the assumption of Rademacher prior on $u$ is not
essential and could be easily generalized to other isotropic distributions,
\cref{assumptions:g} is a key requirement to carry out the proofs.  We discuss
the limitations of the random feature analysis, and in particular the Gaussian
equivalence theorem~\citep{goldt2020modeling, gerace2020generalisation, hu2022universality, mei2022generalization, goldt2022gaussian}, at the end of \cref{sec:experiments-wishart}.

As is generally the case in high-dimensional statistics, our results do not
constitute a complete proof of the existence of a statistical-to-computational
gap: our argument relies on the low-degree \cref{conjecture:LDLR}, and the
divergence of the norm of the likelihood ratio is only a necessary condition for
\emph{strong distinguishability}. In fact, there are no techniques at the moment
that can prove that average-case problems require super-polynomial time in the
hard phase, even if we assume $P \neq NP$ \citep{kunisky2019notes}.  So our work
should be seen as providing rigorous evidence for a statistical-to-computational
gap in the spiked cumulant model.

\section{Details on the replica analysis}%
 \label{app:replicas}
\begin{figure*}[t!]
  \centering
  \includegraphics[width=\linewidth]{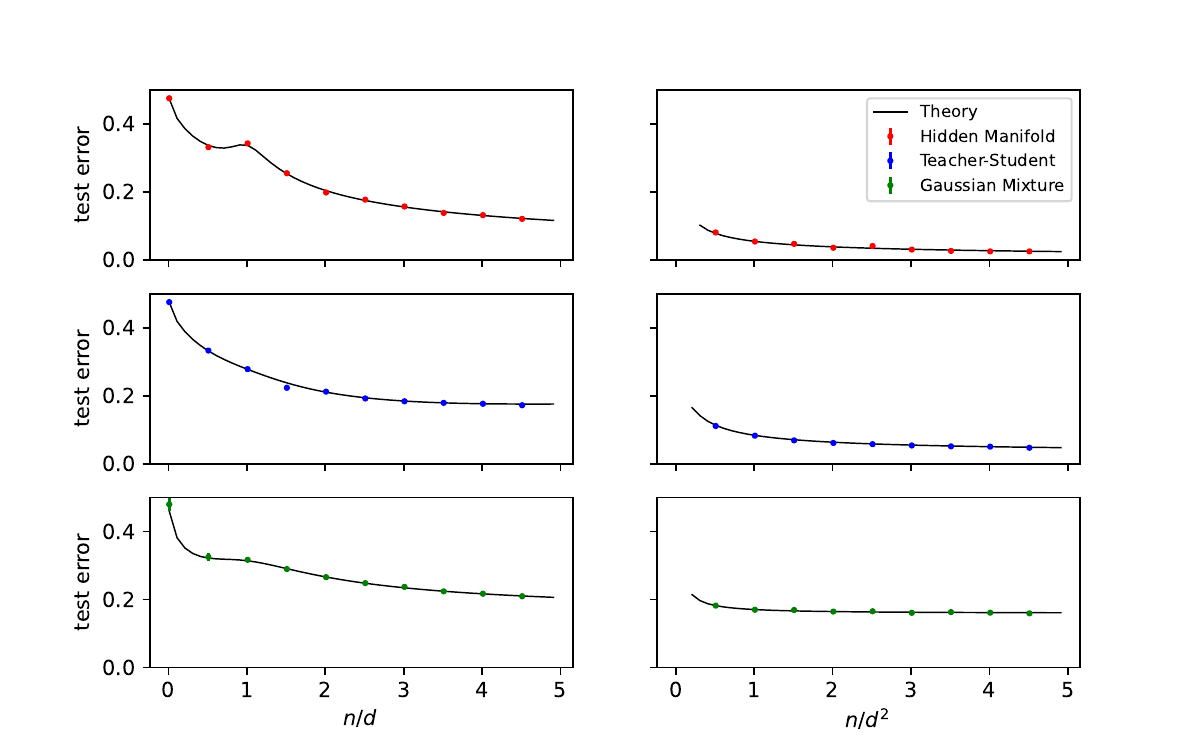}
  \caption{\label{fig:RF_sample_regimes} \textbf{Linear and quadratic sample regimes for different synthetic data models.} (\emph{Right}) Generalization error of the hidden manifold model (top), the teacher-student setup (center) and a mixture of two Gaussians as a function of the ratio of the number of samples and the input dimension. (\emph{Left}) Same except that the number of samples scales with the square of the input dimension. The solid black line corresponds to the replica theory prediction while the coloured dots display the outcome of the numerical simulations averaged over $10$ different seeds. In all the panels, $d = 1000$ and $d = 20$ for linear and quadratic sample regimes respectively, $\lambda = 0.01$ for both the teacher--student setup and the hidden manifold model while $\lambda = 0.1$ for Gaussian mixtures. In the Gaussian mixture case, $\mu_{\pm} = \left( \pm 1, 0, ..., 0 \right) \in \mathbb{R}^d$ while the covariance matrices are both isotropic and, in particular, both equal to the identity matrix: $\Sigma_{\pm} = \mathbb{I}$.
    }
\end{figure*}
We analytically compute the generalisation error of a random feature model
trained on the Gaussian mixture classification task of \ref{sec:wishart} by
exploiting the Gaussian equivalence theorem for dealing with structured data
distributions~\citep{goldt2020modeling, gerace2020generalisation,
  hu2022universality, mei2022generalization, goldt2022gaussian}. In particular,
we use equations derived by \citet{loureiro2021learninggaussians} that describe
the generalisation error of random features on Gaussian mixture classification
using the replica method of statistical physics. The equations describe the
high-dimensional limit where $n, d$, and the size of the hidden layer
$p \rightarrow \infty$ while their ratios stay finite, $\alpha = n/d$ and
$\gamma = d/p \sim O\left( 1\right)$. In this regime, the generalisation error
is a function of only scalar quantities, i.e.\
\begin{equation}
    \epsilon_g = 1 - \rho_+ \mathbb{E}_{\xi} \left[ \mbox{sign}\left( m_+^{\star} + \sqrt{q_+^{\star}}\xi + b^{\star}\right) -  \right] + \rho_- \mathbb{E}_{\xi} \left[ \mbox{sign}\left( m_-^{\star} + \sqrt{q_-^{\star}}\xi + b^{\star}\right)\right],
\end{equation}
where $\xi \sim \mathcal{N}\left(0, 1\right)$, $\rho_{\pm}$ defines the fraction
of data points belonging to class $y=\pm1$ while $m^{\star}_{\pm}$,
$q^{\star}_{\pm}$ and $b^{\star}$ are the so-called overlap parameters and bias
term respectively. Their value for every size of the training set can be
determined by solving the following optimisation problem~\citep{loureiro2021learninggaussians}:
\begin{gather}
  \label{eq:free_energy}
  f_{\beta} =
  \underset{\{q_k,m_k,V_k,\hat{q}_k,\hat{m}_k,\hat{V}_k,b\}}{\mbox{extr}}_{k=+,-}
  \left[\sum_{k = +, -} -\frac{1}{2}\left(\hat{q_k}V_k - q_k\hat{V_k}\right) +
    m_k\hat{m}_k + \underset{p\rightarrow \infty}{\mbox{lim}} \frac{1}{p}
    \Psi_{s} + \alpha \gamma \Psi_{e} \right].
\end{gather}
where the entropic potential $\Psi_s$ can be expressed as a function of the
means $\boldsymbol{\mu}_{\pm} \in \mathbb{R}^p$ and covariances
$\Sigma_{\pm} \in \mathbb{R}^{p \times p}$ of the hidden-layer activations of
the random features model, while the energetic potential $\Psi_e$ depends on the
specific choice of the loss function $\ell \left( \cdot \right)$.

Solving the optimization problem in \ref{eq:free_energy}, leads to a set of coupled saddle-point equations for the overlap parameters and the bias term. This set of equations is nothing but a special case of the equations already derived in \cite{loureiro2021learninggaussians}, that is high-dimensional classification of a mixture of two Gaussians with random features, except for the typo in the self-consistency equation of the bias term

\begin{equation}
    b = \left( \sum_{k =\pm} \frac{\rho_k}{V_k}\right)^{-1}\sum_{k =\pm} \rho_k \mathbb{E}_\xi \left[ \eta_k - m_k \right],
\end{equation}

with $\eta_k$ being the extremizer

\begin{equation}
    \eta_k = \underset{\lambda}{\mbox{argmin}}\left[ \frac{\left( \lambda -  \sqrt{q}_k \xi - m_k - b\right)^2}{2 V_+} + \ell\left(y_k,\lambda \right)\right].
\end{equation}

and $\ell\left( \cdot \right)$ being any convex loss function. 
Analogously to the more general setting of \cite{loureiro2021learninggaussians}, the Gaussian Equivalence Theorem allows to express the entropic potential $\Psi_s$ as a function of the means $\boldsymbol{\mu}_k \in \mathbb{R}^p$ and covariances $\Sigma_k \in \mathbb{R}^{p \times p}$ of the random features

\begin{equation}
\begin{split}
    \Psi_s &= \underset{p\rightarrow
        \infty}{\mbox{lim}} \frac{1}{p} \left(\hat{q}_+\boldsymbol{\mu}_+ + \hat{q}_-\boldsymbol{\mu}_-\right)^t \left( \lambda \mathbb{I}_p + \hat{V}_+ \Sigma_+ + \hat{V}_- \Sigma_- \right)^{-1}\left(\hat{q}_+\boldsymbol{\mu}_+ + \hat{q}_-\boldsymbol{\mu}_-\right)+\\
        &+ \underset{p\rightarrow
        \infty}{\mbox{lim}} \frac{1}{p} \mbox{Tr}\left( \left( \hat{q}_+ \Sigma_+ + \hat{q}_- \Sigma_-\right)\left( \lambda \mathbb{I}_p + \hat{V}_+ \Sigma_+ + \hat{V}_- \Sigma_- \right)^{-1}\right);
\end{split}
\end{equation}

while the energetic potential $\Psi_e$ depends on the specific choice of the loss function $\ell \left( \cdot \right)$

\begin{equation}
\begin{split}
    \Psi_e &= -\frac{1}{2} \mathbb{E}_{\xi} \left[ \sum_{k=\pm}\rho_k \frac{\left(\eta_k - \sqrt{q}_k \xi - m_k - b\right)^2}{2V_k} + \ell\left(y_k,\eta_k \right) \right].
\end{split}
\end{equation}

As discussed in \ref{sec:experiments-wishart}, the Gaussian Equivalence Theorem breaks in the quadratic sample regime if the data are sampled from the spiked Wishart model. Interestingly, this is not the case for the Hidden Manifold model. Indeed, as shown in \cref{fig:RF_sample_regimes}, we still get a perfect match between the Gaussian theory and the numerical simulations even in the quadratic sample regime. The theoretical prediction can be easily derived by rescaling the free-energy associated to the Hidden Manifold model as in \cite{gerace2020generalisation} by a factor $1/d^2$. This is the same trick already proposed in \cite{dietrich1999statistical} for teacher--student settings and displayed in the middle panel of \cref{fig:RF_sample_regimes}.

\end{document}